\documentclass{article}

\usepackage{amsmath}
\usepackage{amssymb}
\usepackage{mathtools}
\usepackage{amsthm}
\usepackage[font=small,labelfont=bf]{caption}
\usepackage{enumitem}
\usepackage{float}
\usepackage{wrapfig} 
\usepackage{comment}

\newcommand \term[1]{\mathtt{term}~(\mathtt{#1})}

\newcommand{\Reg}{\mathrm{Reg}}

\newcommand{\tr}{\mathsf{ T}}

\DeclareMathOperator*{\argmin}{\arg\!\min}

\usepackage{microtype}
\usepackage{graphicx}
\usepackage{subfigure}
\usepackage{booktabs} 

\usepackage{hyperref}

\usepackage[capitalize,noabbrev]{cleveref}

\usepackage{algorithm}
\usepackage{algpseudocode}

\theoremstyle{plain}
\crefname{equation}{}{}
\newtheorem{theorem}{Theorem}
\newtheorem{proposition}{Proposition}
\newtheorem{lemma}{Lemma}
\newtheorem{corollary}{Corollary}
\theoremstyle{definition}
\newtheorem{definition}{Definition}
\newtheorem{assumption}{Assumption}
\theoremstyle{remark}
\newtheorem{remark}{Remark}
\theoremstyle{example}

\theoremstyle{fact}

\usepackage{caption}

\usepackage[preprint]{neurips_2026}

\usepackage[utf8]{inputenc} 
\usepackage[T1]{fontenc}    
\usepackage{hyperref}       
\usepackage{url}            
\usepackage{booktabs}       
\usepackage{amsfonts}       
\usepackage{nicefrac}       
\usepackage{microtype}      
\usepackage{xcolor}         

\title{Online Nonstochastic Prediction: Logarithmic Regret via Predictive Online Least Squares}

\author{Chih-Fan Pai$^{1}$ \quad Yang Zheng$^{1}$\\
  \\
  $^{1}$Department of Electrical and Computer Engineering, UC San Diego
}

\begin{document}

\maketitle

\vspace{-5mm}
\begin{abstract} 
\vspace{-2mm}
\looseness=-1
    We study online prediction for marginally stable, partially observed linear dynamical systems under nonstochastic disturbances.
    Our objective is to minimize~the~cumulative squared prediction loss and compete with the best-in-hindsight Luenberger predictor.
    Standard online learning methods typically rely on bounded domains/gradients, and thus their guarantees may fail to deal with potentially~unbounded trajectories in marginally stable systems.
    In this paper, we introduce an unconstrained online least squares method that stabilizes the learning process via tailored predictive hints. 
    With model knowledge, we prove that hints constructed from any stabilizing Luenberger predictor render the hint residuals uniformly bounded, achieving logarithmic regret despite unbounded trajectory growth.
    We also discuss model-free prediction and introduce a simple universal hint for symmetric systems, under which logarithmic regret is maintained without model knowledge.
    Our results provide an adaptive, instance-wise optimal online predictor compared to classical fixed-gain observers under nonstochastic~disturbances.
\end{abstract}

\vspace{-3mm}
\section{Introduction}
\label{sec:intro}
\vspace{-1mm}

\looseness=-1
This paper studies online prediction in \textit{partially observed} linear dynamical systems (PO-LDS) 
\begin{equation}
    \label{eq:po-lds}
     x_{t+1}=Ax_t+w_t,\quad y_t=Cx_t+v_t,
\end{equation}
with latent state $x_t \in \mathbb{R}^{n}$ and \textit{nonstochastic} disturbances $w_t \in \mathbb{R}^{n}, v_t \in \mathbb{R}^{p}$.
We allow the transition matrix $A \in \mathbb{R}^{n \times n}$ to be \textit{marginally stable}, i.e., $\rho(A) \le 1$, so trajectories $x_t, y_t$ can grow polynomially over time.
The learner observes only $y_t$ sequentially and must predict each $y_t$ before it is revealed.

\looseness=-1
Prediction in LDS is foundational to control, forecasting, and signal processing.
The classical Kalman filter is optimal under Gaussian noise with known statistics~\cite{kalman1960new}; the $\mathcal{H}_\infty$ filter provides a worst-case alternative~\cite{simon2006optimal}. 
Both may be viewed as instances of the \textit{Luenberger predictor}, which is a closed-loop recursive observer parameterized by a suitable Luenberger gain \cite{Luenberger1964Observing}.
In modern applications, however, disturbances are often structured, nonstochastic, and mismatched to either the Kalman or $\mathcal{H}_\infty$ filter's assumptions. 
Accordingly, any single pre-committed gain can be far from optimal on the realized disturbance sequence.
We therefore adopt the \textit{regret minimization} approach and compete against the \textit{best-in-hindsight} stabilizing Luenberger predictor.
This benchmark provides an \textit{instance-wise} optimal predictor that adapts to the realized $w_t, v_t$, which we call \textit{online nonstochastic prediction}, analogous to the best-in-hindsight stabilizing feedback controller in online nonstochastic control~\cite{agarwal2019online,simchowitz2020improper,hazan2022introduction}.

Prior work on online LDS prediction divides along two categories: (a) whether the data is assumed to come from a true LDS (\textit{generative}) or treated as arbitrary (\textit{agnostic}), and (b) whether the comparator is an \textit{open-loop} LDS rollout or a \textit{closed-loop} predictor incorporating feedback.
In the agnostic setting, spectral filtering \cite{hazan2017learning,hazan2018spectral} obtains sublinear regret against the best fitted open-loop LDS under marginal stability by handling long impulse-response memory, and \cite{kozdoba2019line} competes against the Kalman predictor. 
However, both rely on
standard online gradient descent and thus require bounded observations or increments.
In the generative setting under \textit{stochastic} noise, unconstrained \textit{online least squares} (OLS) has been shown to achieve logarithmic regret against the Kalman filter even for marginally stable systems \cite{ghai2020no,tsiamis2022online,rashidinejad2020slip,qian2025model,qian2025logarithmic}.
Extending this to adversarial noise has remained open: the polynomial growth of observations makes squared-loss gradients diverge so regret bound from standard online learning becomes vacuous, and the sharpest existing bound from unconstrained OLS is only sublinear and restricted to the fully observed state-prediction case with $y_t = x_t$ \cite{ghai2020no}. 
An important question arises: \textit{Can we achieve polylogarithmic regret against the best-in-hindsight closed-loop predictor for partially observed, marginally stable LDS under nonstochastic disturbances?}

\looseness=-1
\textbf{Our Contribution.} We answer this affirmatively via \textit{Finite-Memory Predictive Online Least Squares} (FM-POLS), an unconstrained online regression framework for competing with the best stabilizing Luenberger predictor in hindsight.
The key idea draws on a classical insight from {predictive} online learning: when partial information about the upcoming loss is available before the learner commits, regret can scale with what is \textit{unpredictable} rather than with the full signal.
We bring this principle to online nonstochastic prediction by giving the learner a \textit{predictive hint} $\tilde y_t$ (as a guess for $y_t$) before each prediction.
By incorporating this hint, together with the feature $z_t$ available from past observations, into its least-squares update, FM-POLS's regret then scales with the hint residual $\|y_t-\tilde y_t\|$ rather than the algorithm's own prediction error as in standard OLS.
This \textit{exogenous rescaling} effectively reduces the challenge of unbounded trajectories to that of designing hints with bounded residuals.

Within FM-POLS, we design hints in two cases. \textit{With model knowledge}, we show that \emph{any} stabilizing Luenberger observer, even a poorly tuned one, yields a uniformly bounded residual, giving polylogarithmic regret against the best comparator in hindsight. 
\textit{Without model knowledge}, we show that simple polynomial filters of past observations can algebraically remove the marginal stable modes of the system.
In particular, a universal two-lag filter $\tilde{y}_t = y_{t-2}$ suffices for diagonalizable systems with real spectrum (including all symmetric systems with $A = A^\tr$), and higher-order differencing handles real-spectrum Jordan blocks.
The model-free case for complex marginal eigenvalues remains open.
To our knowledge, this is the first logarithmic regret guarantee for online nonstochastic prediction in partially observed, marginally stable~LDS against closed-loop Luenberger predictors.

\looseness=-1
\textbf{Related Work.}
Our formulation parallels \textit{Online Nonstochastic Control}~\cite{agarwal2019online, simchowitz2020improper, hazan2022introduction}, which competes against the best stabilizing feedback controller in hindsight under bounded adversarial disturbances. The closed-loop benchmark, improper learning via finite-memory truncation, and per-step bounded noise model are direct analogs, and our work develops the corresponding tools for online prediction.
In the control setting, across stochastic and adversarial noise models, a clear picture has emerged for when logarithmic regret is achievable: with \textit{known} dynamics,~\cite{agarwal2019logarithmic, foster2020logarithmic,simchowitz2020improper} establish polylogarithmic regret under strongly convex or quadratic costs, tightening the ${\mathcal{O}}(\sqrt{T})$ bounds of~\cite{cohen2018online,agarwal2019online}; with \textit{unknown}~dynamics, polylogarithmic regret is restricted to specific scenarios~\cite{cassel2020logarithmic}  and ruled out in general by $\Omega(\sqrt{T})$ lower bounds~\cite{cassel2020logarithmic,simchowitz2020naive}.
The prediction landscape is less developed: polylogarithmic regret against the best stabilizing predictor in hindsight under adversarial noise and marginal stability has not been established even for known systems. Our work provides such regret guarantees.

\textit{Logarithmic Regret in (Predictive) Online Learning.} Achieving logarithmic regret in standard online optimization typically requires either strong convexity of the loss, exploited by online gradient descent, or exp-concavity, handled by online Newton step~\cite{hazan2007logarithmic, hazan2016introduction,orabona2019modern}.
Both methods see the loss only after committing to its prediction, and rely on projection onto a bounded feasible set to control gradient magnitudes. However, this bounded assumption is incompatible with marginally stable PO-LDS whose observations grow unboundedly.
For square-loss regression, {unconstrained} methods remove the projection requirement: standard OLS attains logarithmic regret whenever prediction errors stay bounded~\cite{cesa2006prediction}, and the Vovk–Azoury–Warmuth (VAW) forecaster~\cite{vovk2001competitive,azoury2001relative} sharpens this by exploiting the feature revealed before the learner commits.
Yet their regret bounds remain coupled to the signal magnitude and thus degrade under marginal stability. 
The framework of \textit{predictive} (or \textit{optimistic}) online learning further allows a {predictive hint} of the upcoming loss function, yielding regret bounds that scale with the hint quality~\cite{rakhlin2013online,rakhlin2013optimization,jacobsen2024online}.
Our FM-POLS may be viewed as a predictive version of OLS (or VAW) for online nonstochastic prediction of LDS.

\textit{Pre-Filtering and Differencing.} While our algorithmic framework follows predictive online learning, the predictive hints we design relate to signal pre-processing for LDS.
There are several forms of pre-filtering observations to remove unstable or long-memory dynamics before learning.
For partially observed marginally stable LDS, \cite{simchowitz2019learning} shows that pre-filtered least squares yields stronger guarantees for system identification under adversarial noise; \cite{marsden2025universal} introduces universal sequence preconditioning for LDS prediction, achieving hidden-dimension-free guarantees.
In classical time-series analysis, the ARIMA family~\cite{box2015time} applies differencing operators at varying lags and orders to render a series stationary before model fitting, though identifying the right orders is itself a hard problem.
Our polynomial-filter hints go beyond this: we design filter coefficients to \textit{algebraically annihilate} the marginal modes of $A$, and establish \textit{online regret} guarantees in adversarial LDS prediction.


\textbf{Organization.}
\cref{sec:problem_setup} formalizes the prediction task and the Luenberger benchmark. 
\cref{sec:pols_framework} develops the POLS framework and specializes it to LDS prediction as FM-POLS, establishing the residual-scaled Luenberger regret bound.
\cref{sec:hint_design} designs predictive hints in both model-based and model-free regimes. 
\cref{sec:experiments} reports numerical experiments and \cref{sec:conclusion} concludes the paper. Detailed proofs and additional experiments are deferred to the appendix.

\textbf{Notation.}
For a matrix $M$, $\|M\|$ and $\|M\|_F$ denote its operator and Frobenius norms, and $\rho(M)$ the spectral radius (maximum absolute eigenvalue); for a vector $y$, $\|y\|$ is the Euclidean norm. We use $[T] := \{1, \ldots, T\}$, $y_{t_1:t_2}$ for the sequence $y_{t_1}, \ldots, y_{t_2}$ with $t_1 \le t_2$, and $\operatorname{poly}(\cdot)$ for a polynomial function of its arguments. For~\cref{eq:po-lds}, we adopt the convention $y_s := 0$ for $s \le 0$, $w_s := 0$ for $s < 0$, and $v_s := 0$ for $s \le 0$.
A list of paper-specific notation is provided in~\cref{tab:notation} in the appendix.

\vspace{-1.5mm}
\section{Problem Formulation}
\label{sec:problem_setup}
\vspace{-2mm}

We formalize the prediction task, the comparator class, and the regret objective for PO-LDS in \cref{eq:po-lds}.

\begin{assumption}[System Regularity]
\label{assum:main}
(i)~$\rho(A) \le 1$;\; 
(ii)~$A = SJS^{-1}$ with Jordan blocks of size $\le r$ and $\kappa_A := \|S\| \|S^{-1}\|$;\;
(iii)~$(A, C)$ is observable;\;
(iv)~$x_0=0$, $\|w_t\| \le C_w$, $\|v_t\| \le C_v$ for all $t$.
\end{assumption}
\vspace{-1mm}

The assumption $x_0 = 0$ is without loss of generality; any bound on the initial state $\|x_0\|\leq C_0$ can simply be absorbed into a modified disturbance bound $C_w$.
The largest Jordan block controls the long-term memory of the system, as stated below (see Appendix~\ref{app:poly-growth} for a proof):

\vspace{0.5mm}
\begin{lemma}[Polynomial Growth]
\label{lem:poly-growth}
Under \cref{assum:main}, for all $k\ge0$ and $t \ge 1$, the following holds:
(i)~$\|A^k\| \le \kappa_A(1+k)^{r-1}$;\;
(ii)~$\sum_{s=0}^{k}\|A^s\| \le \kappa_A(1+k)^r$;\;
(iii)~$\|x_t\| \le C_x (1+t)^r$, $\|y_t\| \le C_y(1+t)^r$, where $C_x := \kappa_A C_w$ and $C_y := \|C\|C_x + C_v$.
\end{lemma}

\vspace{-0.5mm}
\textbf{Interaction Protocol.}
We consider $T$-round prediction. At each  $t\in[T]$, first, the learner predicts $\hat{y}_t$ from historical observations $y_{1:t-1}$; then, the adversary selects $w_{t-1},$ $v_t$ \textit{arbitrarily} subject to \cref{assum:main}-(iv); finally, the true $y_t$ is revealed and the learner incurs squared loss $\|\hat{y}_t - y_t\|^2$.

\textbf{Competing with Luenberger Predictors.} 
We first define our benchmark (closed-loop) predictors.

\begin{definition}[Stabilizing Luenberger Predictor]
For a gain $L \!\in\! \mathbb{R}^{n \times p}$ such that $A_L\!:=\!A-LC$ satisfies $\rho(A_L)\!<\!1$, the associated \textit{Luenberger predictor} maintains an internal state  $\hat{x}_t^{\mathrm{LB}}(L)$ and predicts
\begin{equation} \label{eq:luen-dynamics}
    \hat{y}_{t}^{\mathrm{LB}}(L) = C \hat{x}_t^{\mathrm{LB}}(L), \quad \hat{x}_{t+1}^{\mathrm{LB}} = A \hat{x}_{t}^{\mathrm{LB}} + L(y_{t} - \hat{y}_{t}^{\mathrm{LB}}) = A_L \hat{x}_{t}^{\mathrm{LB}} + L y_{t}, \quad \hat{x}^{\mathrm{LB}}_{0}=0.
\end{equation}
\end{definition}

We adopt the standard notion of \textit{strong stability} \cite{cohen2018online,simchowitz2020improper} to quantify the decay rate of $A_L$.

\vspace{0.5mm}
\begin{definition}[Strong Stability]
We say that $A_L$ is $(\kappa, \gamma)$-strongly stable for some $\kappa > 0$, $\gamma \in (0, 1]$ if there exist matrices $P$, $Q \in \mathbb{R}^{n \times n}$ such that $A_L = P Q P^{-1}$, $\|Q\| \le 1-\gamma$, and $\|P\|\,\|P^{-1}\| \le \kappa$.
\end{definition}

\vspace{-1mm}
Strong stability immediately yields $\|A_L^k\| \le \kappa(1-\gamma)^k$ for all $k \ge 0$, which is useful in our subsequent analysis.
Given $(\kappa,\gamma)$, we define the comparator class
\[
    \mathcal{L}_{\kappa, \gamma}
    := \bigl\{L \in \mathbb{R}^{n \times p} : \|L\|\leq\kappa,\;
       A_L \text{ is } (\kappa, \gamma)\text{-strongly stable}
       \bigr\}.
\]
Under \cref{assum:main}-(iii), the set $\mathcal{L}_{\kappa, \gamma}$ is nonempty for any $\gamma\in(0,1]$ and sufficiently large $\kappa$.
The learner's objective is to minimize the \textit{Luenberger regret} over a horizon $T$, defined as the excess cumulative loss relative to any Luenberger comparator $\hat{y}_t^{\mathrm{LB}}(L)$ induced by $L \in \mathcal{L}_{\kappa,\gamma}$
\begin{equation} \label{eq:luen-reg}
    \Reg_T^{\mathrm{LB}}(L) := \sum_{t=1}^T \|\hat{y}_t - y_t\|^2 - 
    \sum_{t=1}^T \|\hat{y}_t^{\mathrm{LB}}(L) - y_t\|^2.
\end{equation}
Notably, the learner aims to compete with $L^\star \in \argmin_{L \in \mathcal{L}_{\kappa,\gamma}} \!\textstyle\sum_{t=1}^T \|\hat{y}_t^{\mathrm{LB}}(L) - y_t\|^2$, the optimal comparator selected in hindsight after all disturbances are realized. 
Classical Kalman and $\mathcal{H}_\infty$ filters correspond to Luenberger predictors with specific gains. Therefore, competing against the best gain in hindsight provides an instance-wise optimal alternative to any single pre-committed filter. 
We also note that the Luenberger predictor \cref{eq:luen-dynamics} is nonconvex in $L$ and has infinite memory: unrolling \cref{eq:luen-dynamics} gives $\hat{y}_t^{\mathrm{LB}}(L) \!=\! \sum_{k=0}^{t-1} C A_L^k L\, y_{t-k-1}$, with all past observations coupled through powers of $A_L$.

\section{Predictive Online Least Squares for LDS}
\label{sec:pols_framework}

\vspace{-2mm}

This section introduces \textit{Finite-Memory Predictive Online Least Squares} ({FM-POLS}), an unconstrained online regression framework to compete with the Luenberger benchmark.
The construction of predictive hints will be discussed in \cref{sec:hint_design}.

\subsection{Online Least Squares with Predictive Hints}
\label{subsec:pols}
\vspace{-2mm}

We begin with a general online regression setting.
At each step $t \in [T]$, the learner receives a feature $z_t \in \mathbb{R}^{d}$, predicts $\hat{y}_t = M_t z_t$, and then, observes $y_t$ and incurs squared loss $\|\hat{y}_t - y_t\|^2$. 
Performance is measured by the \textit{regression regret} against any fixed linear predictor $M \in \mathbb{R}^{p \times d}$, defined as
\[
    \Reg_T(M) = \sum_{t=1}^T \|M_t z_t - y_t\|^2 - \sum_{t=1}^T \|M z_t - y_t\|^2.
\]
Note that the comparator $M$ may be chosen in hindsight with full knowledge of $(z_t, y_t)_{t=1}^T$.

\textbf{Standard OLS and its Limitation.}
At each $t$, the standard OLS strategy minimizes the regularized past cumulative loss
$
M_t^{\mathrm{ols}}
\!=\!
\argmin_{M}
\left\{
\lambda \|M\|_F^2
+
\sum_{s=1}^{t-1}
\|M z_s - y_s\|^2
\right\}
$
and predicts $\hat{y}_t \!=\! M_t^{\mathrm{ols}} z_t$.
A classical analysis \cite{cesa2006prediction,ghai2020no} shows that OLS satisfies
\begin{equation}
\Reg_T(M)
\le
{\lambda}\|M\|_F^2 + \max_{1\le t\le T}\|M_t^{\mathrm{ols}} z_t-y_t\|^2 d \log\left(\!1+\frac{\sum_{t=1}^T\|z_t\|^2}{\lambda d}\!\right).
\label{eq:ols_regret}
\end{equation}
This upper bound directly yields logarithmic regret whenever the \textit{prediction error} $\|M_t^{\mathrm{ols}} z_t - y_t\|$ are uniformly bounded. This is indeed the case when both feature $z_t$ and observation $y_t$ are bounded and $M_t^{\mathrm{ols}}$ is restricted to a bounded set.   
For marginally stable LDS, however, $z_t$ and $y_t$ may contain polynomially growing observations, making the prediction error difficult to control and failing the logarithmic regret even in the fully observed case \cite{ghai2020no}.

\textbf{POLS: from Prediction Error to Hint Residual.}
To overcome this, we introduce a \textit{predictive hint} $\tilde{y}_t \in \mathbb{R}^{p}$ that acts as a prior guess of $y_t$.
With $\tilde{y}_t,z_t$ available before the learner commits to $\hat{y}_t$, POLS (predictive OLS) augments the OLS objective with the \textit{look-ahead} term $\|M z_t - \tilde{y}_t\|^2$:
\begin{equation}
    M_t^{\mathrm{pols}} = \argmin_{M \in \mathbb{R}^{p \times d}} \left\{ \lambda \|M\|_F^2 + \sum_{s=1}^{t-1} \|M z_s - y_s\|^2 + \|M z_t - \tilde{y}_t\|^2 \right\},
    \label{eq:pols-objective}
\end{equation}
and the learner predicts $\hat{y}_t = M_t^{\mathrm{pols}} z_t$. 
The key consequence is that the regret now depends on the \textit{hint residual} $\Delta_t := y_t - \tilde{y}_t$ rather than on the prediction error. We have the following guarantee. 

\begin{theorem}[Residual-Scaled Regret] \label{thm:pols-reg}
For any $\lambda > 0$, any $M$, and any $\{z_t, y_t\}_{t=1}^T$, POLS guarantees
\begin{equation}
\Reg_T(M) \le \lambda \|M\|_F^2 + \max_{1\le t\le T}\|y_t - \tilde{y}_t\|^2 d \log\left(\! 1+\frac{\sum_{t=1}^T \|z_t\|^2}{\lambda d} \!\right).
\label{eq:pols_regret}
\end{equation}
\end{theorem}
Comparing \cref{eq:ols_regret} and \cref{eq:pols_regret}: OLS scales with the prediction error $\|M_t^{\mathrm{ols}} z_t - y_t\|$, which is an \textit{endogenous} quantity that depends on the algorithm's own iterates and may diverge; POLS scales with the hint residual $\|y_t - \tilde{y}_t\|$, which is \textit{exogenous} and depends only on the quality of the hint. This decoupling is what enables logarithmic regret even when signals grow unboundedly.

The proof of \cref{thm:pols-reg} is provided in Appendix~\ref{app:reg-analysis}, adapting the techniques from \cite{orabona2015generalized,foster2020logarithmic,jacobsen2024online}.
Note that POLS strictly generalizes OLS: setting the self-consistent hint $\tilde{y}_t = M_t^{\mathrm{pols}} z_t$ recovers OLS (see Appendix~\ref{app:ols-as-pols}).
When $\tilde y_t = 0$, POLS reduces to the classical VAW forecaster \cite{vovk2001competitive,azoury2001relative}.
Also, POLS is connected to the undiscounted version of a recent \cite[Algorithm 1]{jacobsen2024online}.

\begin{algorithm}[t]
\caption{FM-POLS: Finite-Memory Predictive Online Least Squares}
\label{alg:fm-pols}
\begin{algorithmic}[1]
\Require Horizon $T$, regularization $\lambda > 0$, memory length $H= \mathcal{O}(\log T)$.
\State \textbf{Initialize:} $M_0 = \mathbf{0} \in \mathbb{R}^{p \times d}$, $P_0 = \lambda^{-1} I \in \mathbb{R}^{d \times d}$, where $d=pH$.
\For{$t = 1, \dots, T$}
    \State \textbf{Receive}: feature $z_t = [y_{t-1}^\tr, \dots, y_{t-H}^\tr]^\tr \in \mathbb{R}^{d}$ and hint $\tilde{y}_t \in \mathbb{R}^{p}$
    \State \textbf{Update Covariance}: $K_t = \frac{P_{t-1} z_t}{1 + z_t^\tr P_{t-1} z_t}$, $P_t = P_{t-1} - K_t z_t^\tr P_{t-1}$
    \State \textbf{Predict}: $\hat{y}_t = M_t^{\mathrm{pols}} z_t$, \text{where} $M_t^{\mathrm{pols}} = M_{t-1} + (\tilde{y}_t - M_{t-1} z_t) K_t^\tr$
    \State \textbf{Observe} $y_t$ and  \textbf{Update}: $M_t = M_{t-1} + (y_t - M_{t-1} z_t) K_t^\tr$
\EndFor
\end{algorithmic}
\end{algorithm}

\subsection{Regression via Finite-Memory Truncation} \label{subsec:truncation}

To apply POLS in LDS prediction, we reduce the prediction task to a finite-dimensional regression via \textit{improper learning} and finite-memory truncation. In particular, we over-parameterize the predictor space using a memory length $H$.
Define the stacked feature for all $t\ge 1$
\begin{equation}
\label{eq:stacked_feature}
    z_t := [y_{t-1}^\tr, y_{t-2}^\tr, \ldots, y_{t-H}^\tr]^\tr \in \mathbb{R}^{pH}. 
\end{equation}
The proposed FM-POLS is POLS instantiated with feature $z_t$ in \cref{eq:stacked_feature} and target $y_t$; \cref{alg:fm-pols} gives the efficient closed-form recursion for FM-POLS. 
The inverse Gram matrix $P_t := G_t^{-1}$, where $G_t := \lambda I + \sum_{s \le t} z_s z_s^\tr$, is maintained recursively via the Sherman-Morrison update.

\vspace{-1mm}
Now we analyze the truncation error. Define the truncated Luenberger predictor (no effect for $t \le H$)
\begin{equation*}
    \hat{y}_t^{\mathrm{TR}}(L)
    := \textstyle\sum_{k=0}^{H-1} C A_L^k L\, y_{t-k-1}
    = M_L z_t,
    \quad
    M_L := [CL,\; CA_LL,\; \ldots,\; CA_L^{H-1}L].
\end{equation*}
The truncation discards the tail $\sum_{k \ge H} CA_L^k L\, y_{t-k-1}$. Since $A_L$ is $(\kappa, \gamma)$-strongly stable, this tail decays geometrically in $H$, fast enough to overcome the polynomial growth of $y_t$.
The following lemma makes this precise, with its proof presented in Appendix \ref{app:trunc}.

\begin{lemma}[Luenberger Truncation Error]
\label{lem:luen-trunc}
Under \cref{assum:main}, for all $t \ge 1$ and $L \in \mathcal{L}_{\kappa,\gamma}$,
\begin{equation*}\label{eq:trunc-ptwise}
    \|\hat{y}_t^{\mathrm{LB}}(L) - \hat{y}_t^{\mathrm{TR}}(L)\|
    \le \gamma^{-1}{\|C\| \kappa^2 C_y} (1+t)^r (1-\gamma)^H,
\end{equation*}
where $C_y$ is defined in \cref{lem:poly-growth}.
Setting $H \ge \lceil {\gamma^{-1}(r+1)\log(1+T)}\rceil$ ensures
    $\sum_{t=1}^T \|\hat{y}_t^{\mathrm{LB}}(L) - \hat{y}_t^{\mathrm{TR}}(L)\|
    \le C_{\mathrm{trun}}$,
   and 
    $\sum_{t=1}^T \|\hat{y}_t^{\mathrm{LB}}(L) - \hat{y}_t^{\mathrm{TR}}(L)\|^2
    \le C_{\mathrm{trun}}^2$,
    where $C_{\mathrm{trun}}:=\gamma^{-1} {\|C\|\kappa^2 C_y}$.
\end{lemma}

We next examine the Luenberger comparator's own prediction error, which is also required in our subsequent regret analysis.
Define the state estimation error $e_t^{\mathrm{LB}}(L) := x_t - \hat x_t^{\mathrm{LB}}(L)$.
The system dynamics \cref{eq:po-lds} and the Luenberger update \cref{eq:luen-dynamics} give the closed-loop recursion
\begin{equation}\label{eq:luen-error-dynamics}
    e_{t+1}^{\mathrm{LB}}(L) = A_L\, e_t^{\mathrm{LB}}(L) + w_t - L v_t,
\end{equation}
yielding the output prediction error $\hat y_t^{\mathrm{LB}}(L) - y_t = C e_t^{\mathrm{LB}}(L) + v_t$.
Again by strong stability of $A_L$, past disturbances are forgotten exponentially, giving the following uniform bound.

\begin{lemma}[Luenberger Prediction Error]
\label{lem:luen-pred}
Under \cref{assum:main}, for all $t \ge 1$ and $L \in \mathcal{L}_{\kappa,\gamma}$, 
\begin{equation*}\label{eq:lb-error}
    \|\hat{y}_t^{\mathrm{LB}}(L) - y_t\|
    \le \gamma^{-1}{\|C\|\kappa(C_w + \kappa C_v)} + C_v
    =: C_{\mathrm{pred}}.
\end{equation*}
\end{lemma}

The proof is given in Appendix~\ref{app:lb-error}. Despite the unbounded growth of observations $y_t=\mathcal{O}(t^r)$, the prediction error stays uniformly bounded under bounded adversarial noise.
The analogous quantity is generally unbounded in the stochastic setting~\cite{tsiamis2022online}.


\vspace{-1mm}
\subsection{From Luenberger Regret to Hint Residual}
\label{subsec:decomposition}
\vspace{-1mm}

We now combine the regression regret of \cref{thm:pols-reg} with the truncation analysis of \cref{lem:luen-trunc,lem:luen-pred} into a residual-scaled bound on the Luenberger regret. 
The key observation is that the unconstrained regression class $\{Mz_t : M \in \mathbb{R}^{p \times pH}\}$ contains \textit{every} truncated Luenberger predictor $\hat{y}^{\mathrm{TR}}_t(L) = M_Lz_t$ for $L \in \mathcal{L}_{\kappa,\gamma}$, so the Luenberger regret \cref{eq:luen-reg} decomposes cleanly:
\begin{equation*}
\label{eq:regret_decomp}
    \Reg_T^{\mathrm{LB}}(L) = {\Reg_T(M_{L})} + {\sum_{t=1}^T \left(\|M_{L} z_t - y_t\|^2 - \|\hat{y}_t^{\mathrm{LB}}(L) - y_t\|^2\right)}.
\end{equation*}

\begin{theorem}[Residual-Scaled Luenberger Regret of FM-POLS]
\label{thm:total-regret}
Under \cref{assum:main}, suppose FM-POLS is run with memory length $H =\lceil\gamma^{-1}(r+1)\log(1+T)\rceil\footnote{This requires knowing $r$; any $\bar r \ge r$ suffices, inflating $H$ roughly by a factor of $\bar r/r$ while reducing the truncation error.}= \mathcal{O}(\log T)$ and any predictive hints $\tilde y_t$. 
Let $\Delta_{\max} := \max_{t\in[T]} \|y_t - \tilde{y}_t\|$.
Then, for any $\lambda > 0$ and any $L \in \mathcal{L}_{\kappa,\gamma}$,
\begin{equation*} \vspace{2pt}
    \Reg_T^{\mathrm{LB}}(L)
    \le
    \frac{\lambda\|C\|^2 \kappa^4 p}{\gamma}
    + \Delta_{\max}^2\, pH
      \log\!\left(\!1 \!+\! \frac{C_y^2(1\!+\!T)^{2r+1}}
      {\lambda p}\!\right)
    + C_{\mathrm{trun}}^2
    \!+ 2C_{\mathrm{pred}}C_{\mathrm{trun}}, \vspace{2pt}
\end{equation*}
where $C_y$, $C_{\mathrm{trun}}$, $C_{\mathrm{pred}}$ are from Lemmas~\ref{lem:poly-growth}, \ref{lem:luen-trunc}, \ref{lem:luen-pred} respectively. 
If $\Delta_{\max} = \operatorname{polylog}(T)$, the regret is polylogarithmic in $T$. In particular, whenever $\Delta_{\max}$ is bounded independently of $T$,
\begin{equation*}
    \Reg_T^{\mathrm{LB}}(L)
    \le
    \operatorname{poly}\left(
    \|C\|, p, r, \kappa_A, \kappa,
    \gamma^{-1}, 
    C_y, C_w, C_v, \Delta_{\max}
    \right)
    \cdot \log^2 T.
\end{equation*}
\end{theorem}

\looseness=-1
The proof is given in Appendix~\ref{app:total-regret}. 
\cref{thm:total-regret} holds uniformly over the class $L \in \mathcal{L}_{\kappa,\gamma}$ and therefore applies to the best comparator in hindsight. 
What remains is a single question: \textit{how fast does $\Delta_{\max}$ grow with $T$?} 
\cref{sec:hint_design} answers it by constructing hints $\tilde{y}_t$ with $\Delta_{\max}$ independent of $T$.
\vspace{1mm}
\looseness=-1
\begin{remark}[Comparison with Prior Work]
\cref{thm:total-regret} extends the logarithmic-regret regime of online PO-LDS prediction from stochastic to adversarial noise.
In the stochastic setting, standard OLS achieves logarithmic regret against the single Kalman predictor~\cite{ghai2020no, tsiamis2022online, qian2025model}; our bound holds for every bounded disturbance sequence and competes against the full Luenberger class. 
The only prior adversarial result~\cite[Thm.\ 3]{ghai2020no} gives sublinear regret for OLS, restricted to the fully observed case $y_t = x_t$ with $T$-dependent regularization; our guarantee in \cref{thm:total-regret} is polylogarithmic, applies to partially observed systems, and holds for any fixed $T$-independent regularization $\lambda > 0$.
\end{remark}

\vspace{-1mm}
\section{Design of Predictive Hints}
\label{sec:hint_design}
\vspace{-1mm}

In this section, we design predictive hints $\tilde{y}_t$ in two cases: model-based hints from a closed-loop Luenberger predictor, and model-free hints from open-loop polynomial filters.

\subsection{Model-Based Hints via Luenberger Predictors}
\label{sec:hint_model}

When $(A, C)$ are known, a natural predictive hint can be constructed by a Luenberger predictor associated with \textit{any} stabilizing gain $\tilde{L}$ such that $A_{\tilde{L}} := A - \tilde{L}C$ is $(\tilde\kappa, \tilde\gamma)$-strongly stable. Specifically,
\begin{equation}\label{eq:hint-observer}
    \tilde{y}_t := C\hat{x}_t^{\mathrm{LB}}(\tilde{L}),
    \quad
    \hat{x}_{t+1}^{\mathrm{LB}}
    = A_{\tilde{L}}\,\hat{x}_t^{\mathrm{LB}} + \tilde{L}\,y_t,
    \quad
    \hat{x}_0^{\mathrm{LB}} = 0.
\end{equation}
We note that, first, the hint gain $\tilde{L}$ is decoupled from the comparator class.
It need not belong to $\mathcal{L}_{\kappa,\gamma}$ and may be conservative (with small $\tilde\gamma$); FM-POLS still adapts online to compete with the best $L \in \mathcal{L}_{\kappa,\gamma}$ in hindsight. 
Second, the hint has constant memory. A single state vector $\hat{x}_t^{\mathrm{LB}}(\tilde{L}) \in \mathbb{R}^n$ suffices for all $t$, in contrast to the $H$-tap predictor used by the learner.

The hint residual $\Delta_t = y_t - \tilde{y}_t$ obeys the same closed-loop error dynamics \cref{eq:luen-error-dynamics} as the Luenberger comparator analyzed in \cref{lem:luen-pred}, with $A_{L}$ in place of $A_{\tilde{L}}$.
Strong stability of $A_{\tilde{L}}$ therefore yields a uniform bound on $\Delta_t$, and combining with \cref{thm:total-regret} directly gives polylogarithmic regret.

\vspace{1mm}
\begin{corollary}[Polylogarithmic Regret for Known Systems]
\label{cor:known-regret}
Under \cref{assum:main}, suppose FM-POLS is run with the memory length $H$ from \cref{thm:total-regret} and the predictive hints $\tilde{y}_t$ in \cref{eq:hint-observer} for any $\tilde{L}$ such that $A_{\tilde{L}}$ is $(\tilde\kappa, \tilde\gamma)$-strongly stable. Then $\Delta_{\max}$ is bounded uniformly in $T$:
\[
    \Delta_{\max}
    \le \tilde\gamma^{-1}{\|C\|\tilde\kappa(C_w + \tilde\kappa\,C_v)} + C_v,
\]
and for any $\lambda>0$ and any $L \in \mathcal{L}_{\kappa,\gamma}$,
\[
    \Reg_T^{\mathrm{LB}}(L)
    \le \operatorname{poly} \left(
    \|C\|, p, r, \kappa_A, \kappa, \tilde\kappa,
    \gamma^{-1}, \tilde{\gamma}^{-1},
    C_w, C_v
    \right) \cdot \log^2 T.
\]
\end{corollary}

\vspace{1.5mm}
\looseness=-1
\begin{remark}[Fully Observed State Prediction]
\label{rem:warmup}
When $C = I$ and $v_t = 0$, the task reduces to predicting $x_{t}$ from $x_{1:t-1}$.
One can take $z_t = x_{t-1}$ and the hint $\tilde{y}_t = Ax_{t-1}$ (the degenerate case of~\cref{eq:hint-observer} with $\tilde{L} = A$), resulting in $\Delta_t = w_{t-1}$ and $\Delta_{\max} \le C_w$. 
Polylogarithmic regret then follows directly from \cref{thm:pols-reg}. 
This strengthens the sublinear OLS regret bound of~\cite{ghai2020no} under perfect knowledge of $A$.
\end{remark}

\subsection{Model-Free Hints via Polynomial Filtering} \label{sec:hint_free}
When $(A, C)$ are unknown, we construct the predictive hint $\tilde{y}_t$ as a linear combination of recent observations.
Given coefficients $c_0 = 1$ and $c_1, \ldots, c_m \in \mathbb{R}$, the associated \textit{polynomial filter} is $q(z):=\sum_{i=0}^m c_i\,z^{m-i}$, with $q(A) = \sum_{i=0}^m c_i A^{m-i}$. 
We adopt the hint
\begin{equation}\label{eq:polynomial-hint}
    \tilde{y}_t := -\sum_{i=1}^m c_i\,y_{t-i}\quad 
    \Longrightarrow \quad
    \Delta_t := y_t-\tilde{y}_t = \sum_{i=0}^m c_i\,y_{t-i},
\end{equation}
so the polynomial coefficients become residual weights on the observations.
The residual $\Delta_t$ admits the following key decomposition; see proof in Appendix \ref{app:residual-decomp}.

\begin{lemma}[Residual Decomposition]
\label{lem:residual-decomp}
Under the predictive hint $\tilde{y}_t$ in \cref{eq:polynomial-hint}, for all $t \ge 1$,
\begin{equation*}\label{eq:residual-decomp}
    \Delta_t
    = \sum_{i=0}^m c_i v_{t-i}
    + C\sum_{s=0}^{m-1}\Bigl(\sum_{i=0}^s c_iA^{s-i}\Bigr) w_{t-1-s}
    + C{\sum_{s=0}^{t-m-1} q(A) A^{s} w_{t-m-1-s}}.
\end{equation*}
\end{lemma}

For fixed $m$, the first two terms form a \textit{short-memory} contribution involving at most $m+1$ bounded disturbances, and contribute $\mathcal{O}(\|q\|_1(C_v + \|C\|\kappa_A C_w))$ with $\|q\|_1 := \sum_{i=0}^m |c_i|$.
The technical challenge lies in the last \textit{infinite-memory} term: when $\rho(A) = 1$, the series $\sum_s \|A^s\|$ diverges, so controlling it requires choosing $q$ such that $q(A)$ annihilates the marginal modes of $A$.\footnote{Polylogarithmic growth of the infinite sum suffices for polylog regret via~\cref{thm:total-regret}; we here focus on exact annihilation.}

\textbf{Exact Spectral Annihilation.}
If the spectrum of $A$ were known, the characteristic polynomial $q_{\mathrm{ch}}$ would satisfy $q_{\mathrm{ch}}(A) = 0$ by Cayley--Hamilton, yielding a bounded residual for any LDS with $\rho(A)\le1$.
Such spectral knowledge is rarely available, and even then $\|q_{\mathrm{ch}}\|_1$ can be exponential in $n$. The useful question is therefore: \textit{for which systems does a fixed, low-order polynomial annihilate the marginal modes without spectral knowledge?} We answer this for real-spectrum systems below.

\subsubsection{The 2-Lag Filter for Real-Spectrum Diagonalizable Systems}
\label{sec:2lag}

For diagonalizable $A$ with real eigenvalues in $[-1, 1]$ (including symmetric $A$ as a special case), the only possible marginal eigenvalues are $\pm1$.
Both are roots of the universal polynomial $q(z) = z^2 - 1$, which therefore annihilates the marginal modes of $A$.
The corresponding hint is the \textit{2-lag filter} $\tilde{y}_t = y_{t-2}$, which uses no system knowledge.

\begin{theorem}[Bounded Residual under $2$-Lag Hints]
\label{thm:2lag}
Suppose \cref{assum:main} holds and $A\in\mathbb{R}^{n\times n}$ is diagonalizable ($r=1$) with real eigenvalues in $[-1, 1]$.
Under the 2-lag predictive hint $\tilde{y}_t = y_{t-2}$,
\begin{equation*}
    \Delta_{\max}
    \le 2C_v+(1+\kappa_A+2n\kappa_A)\|C\|C_w.
\end{equation*}
\end{theorem}

The proof relies on the following key bound, which is interesting by itself. 
The series $\sum_s \|A^s\|$ diverges when $\rho(A) \!=\! 1$, but the prefactor $(A^2-I)$ removes the marginal modes and makes $\sum_s \|(A^2-I)A^s\|$ summable. The proof is provided in Appendix~\ref{app:filtered-jordan-proof}.

\begin{lemma}[Filtered Jordan Bound: Diagonalizable Case]\label{lem:2lag-operator}
Under the assumptions of \cref{thm:2lag},
\[
\textstyle\sum_{s=0}^{\infty}\|(A^2 - I)\,A^s\| \le 2n\kappa_A.
\]
\end{lemma}

\vspace{-3mm}

\begin{proof}[Proof of \cref{thm:2lag}]
With $q(z) = z^2 - 1$ and $m=2$, \cref{lem:residual-decomp} yields
\begin{equation*}
\Delta_t = \underbrace{(v_t - v_{t-2})}_{\mathtt{measurement~noise}}
+ \underbrace{C w_{t-1} + C A w_{t-2}}_{\mathtt{short~memory}}
+ \underbrace{C \textstyle\sum_{s=0}^{t-3} (A^2 - I) A^s w_{t-3-s}}_{\mathtt{infinite~memory}}.
\end{equation*}
Clearly, the magnitudes of the first two terms are at most $2C_v$ and $\|C\|(1 + \kappa_A)C_w$, respectively. 
For the infinite-memory term, applying the triangle inequality and \cref{lem:2lag-operator} gives 
\[
\| C\textstyle\sum_{s=0}^{t-3} (A^2 - I) A^s w_{t-3-s} \| \le \|C\|C_w  \textstyle\sum_{s=0}^{\infty} \|(A^2 - I) A^s \| \le 2n\kappa_A \|C\|C_w.
\]
Combining all three bounds completes the proof.
\end{proof}

\subsubsection{High-Order Differencing for Jordan Blocks}
\label{sec:jordan-hints}

We now consider a class of non-symmetric matrices $A$ with non-trivial Jordan structure. 
Suppose $A$ has real spectrum in $[-1, 1]$ with Jordan blocks of size at most $r\ge2$.
The $2$-lag filter extends to this case by raising the multiplicity of the annihilating polynomial. Specifically, we consider\footnote{This requires knowledge of $r$; an upper bound $\bar r \ge r$ also works but may increase the noise amplification exponentially.}
\begin{equation}\label{eq:high-order-poly}
q_r(z) = (z^2 - 1)^r= \sum_{k=0}^r(-1)^k\binom{r}{k}z^{2(r-k)},
\qquad
\tilde{y}_t = \sum_{k=1}^r (-1)^{k+1} \binom{r}{k} y_{t-2k}.
\end{equation}
In the notation of \cref{eq:polynomial-hint}, $q_r$ has degree $2r$ with coefficients $c_{2k} = (-1)^k \binom{r}{k}$ for $0\le k \le r$, giving $\|q_r\|_1 = 2^r$. The factor $2^r$ represents the noise amplification incurred by higher-order differencing.

\looseness=-1
\begin{theorem}[Bounded Residual under High-Order Differencing]
\label{thm:jordan-hint}
Suppose \cref{assum:main} holds and $A$ has real eigenvalues in $[-1, 1]$. Under the polynomial hint \cref{eq:high-order-poly}, $\Delta_{\max}$ is bounded uniformly in $T$:
\begin{equation}\label{eq:jordan-bound}
    \Delta_{\max}
    \le 2^r C_v + (H_r + K_r n) \kappa_A \|C\| C_w,
\end{equation}
where $H_r$ and $K_r$ are constants depending only on $r$.
\end{theorem}

The proof of \cref{thm:jordan-hint} follows the same procedure as \cref{thm:2lag}; we defer it to Appendix \ref{app:high-order-thm-proof}.
The technical core is the following bound, a generalization of \cref{lem:2lag-operator}, whose proof is much more involved. We present the proof details in  Appendix \ref{app:high-order-proof}.

\begin{lemma}[Gap-Free Filtered Jordan Bound]
\label{lem:high-order-operator}
Under the assumptions of \cref{thm:jordan-hint}, we have
$
\sum_{s=0}^\infty \|(A^2 - I)^r A^s\| \le K_r n \kappa_A,
$
where $K_r$ depends only on $r$, independent of $A$'s spectral gap.
\end{lemma}

\looseness=-1
\begin{remark}[Exponential Dependence on $r$]
\label{rem:cost-of-diff}
The constants in \cref{eq:jordan-bound} satisfy $H_r \le 8^r$ and $K_r \le r(2e^2)^{r}$, both exponential in $r$; see Appendix~\ref{app:explicit-constants} for derivations.
The exponential dependence of $r$ in bounding $\Delta_{\max}$ for this filter appears unavoidable: the factor $\|q_r\|_1 =2^r$ enters the measurement-noise term directly, regardless of how tightly $H_r$ and $K_r$ are bounded. 
In practice, $r$ is not too large. 
\end{remark}


\textbf{Complex Marginal Eigenvalues.}
Our model-free hint constructions exploit that real marginal modes form a finite set $\{\pm 1\}$, so a single fixed polynomial annihilates all such modes.
However, complex marginal eigenvalues $e^{\pm j\theta}$, $\theta \in (0,\pi)$ form a continuum on the unit circle, so no polynomial with fixed coefficients can vanish everywhere; any system-dependent choice such as $q(z) = z^2 - 2\cos\theta \cdot z + 1$ (with roots $e^{\pm j\theta}$) requires knowing $\theta$.
The fundamental difficulty of complex marginal eigenvalues is recognized in model-free LDS prediction: \cite{hazan2017learning,rashidinejad2020slip,marsden2025universal} encounter analogous obstacles in different settings. The fully model-free complex-eigenvalue case remains open here as well; combining polynomial filtering with spectral estimation is a natural direction.

\vspace{-1.5mm}
\section{Experiments}
\label{sec:experiments}

\vspace{-2mm}

\looseness=-1
We report three experiments: showing logarithmic regret (Exp1), comparison with classical fixed-gain filters (Exp2), and parameter sensitivity (Exp3).
Appendix~\ref{app:additional-exp} contains additional experiments. 

\textbf{Noise model and pathwise evaluation.}
We adopt a unified noise model in this section:
\begin{equation}\label{eq:noise-model}
w_t = \bar w + a_w \sin(\omega_w t) + \epsilon_{w,t},
\qquad
v_t = \bar v + a_v \sin(\omega_v t) + \epsilon_{v,t},
\end{equation}
where $\bar w, \bar v$ are fixed biases, $a_w \sin(\omega_w t), a_v \sin(\omega_v t)$ are structured  sinusoidal components, and $\epsilon_{w,t} \!\sim\! \mathrm{Unif}([-C_w, C_w]^{n})$ and $\epsilon_{v,t} \!\sim\! \mathrm{Unif}([-C_v, C_v]^{p})$ are i.i.d. uniform components, all bounded so that \cref{assum:main}-(iv) is satisfied.
Specific amplitudes are chosen per experiment to highlight different regimes: Exp1 and Exp3 use a random-dominant setup; Exp2 uses a bias-sine-dominant setup.
Since our regret guarantee is \emph{pathwise}, holding for each realized noise sequence individually, we run each experiment as a single trial and benchmark against the running best-in-hindsight Luenberger gain at each $t$, $L^\star(t) := \argmin_{L} \sum_{s=1}^{t} \|y_s - \hat y_s^{\mathrm{LB}}(L)\|^2$, computed via grid search.

\begin{figure}[t]
    \centering
    \setlength{\abovedisplayskip}{1pt}
    \includegraphics[width=1\linewidth]{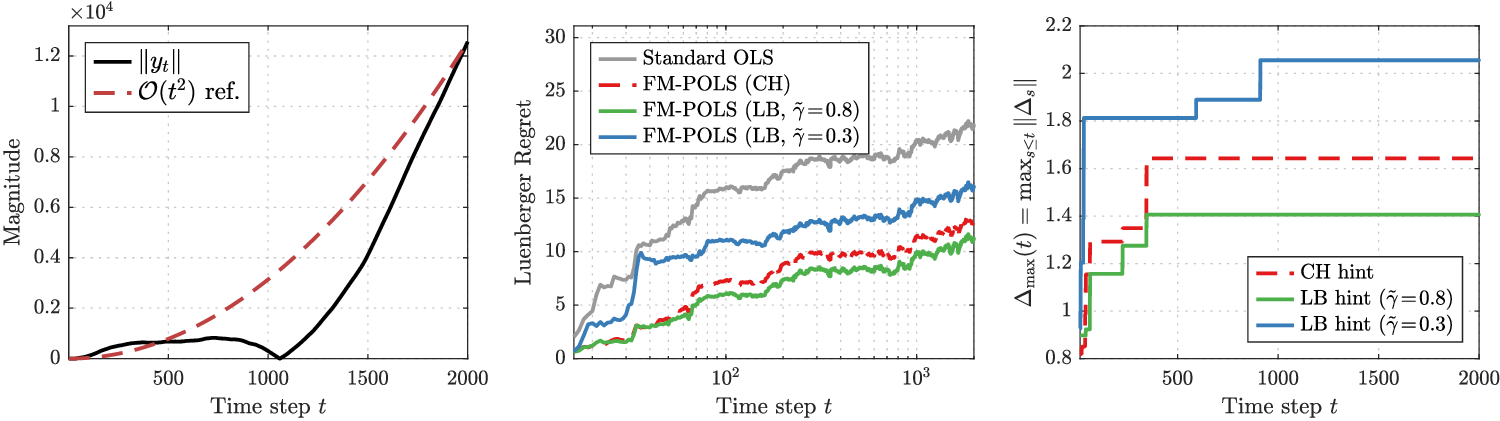}
    \caption{Exp1: Double integrator.
    Left: Polynomial signal growth $\|y_t\| = \mathcal{O}(t^2)$.
    Middle: Luenberger regret grows logarithmically as $\mathcal{O}(\log t)$ with fixed $H$ (log time axis).
    Right: Running max hint residual $\Delta_{\max}(t)$. \vspace{-2mm}}
    \label{fig:exp1}
\end{figure}

\begin{figure}[t]
    \centering
    \includegraphics[width=1\linewidth]{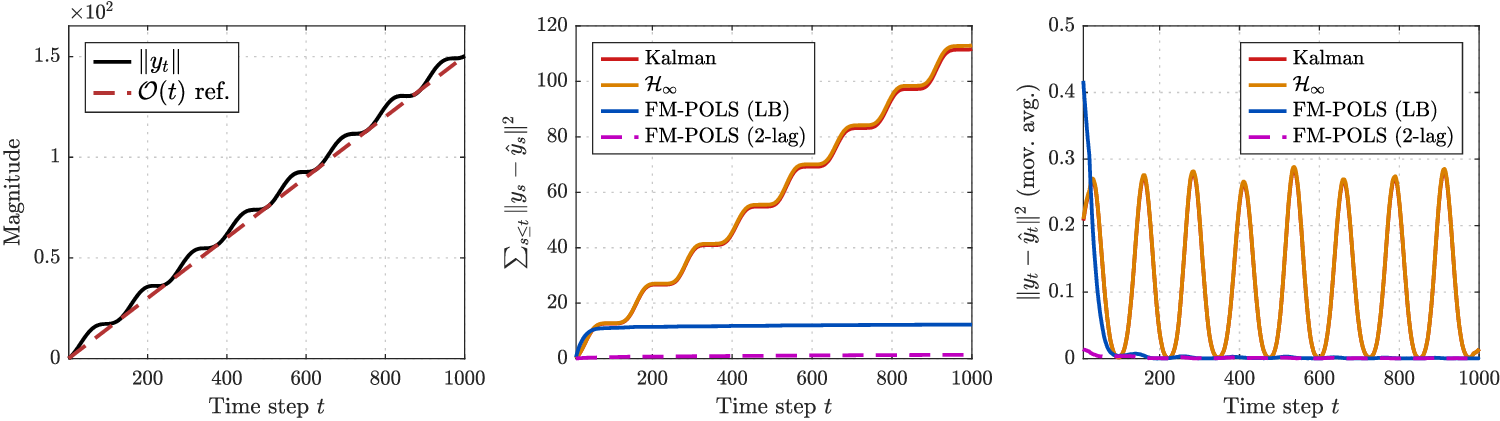}
    \caption{Exp2: Symmetric system.
    Left: Linear signal growth. Middle: Cumulative prediction loss. Right: Per-step loss.
    FM-POLS beats classical fixed-gain filters; the 2-lag hint outperforms the Luenberger hint.
    \vspace{-2mm}}
    \label{fig:exp2}
\end{figure}

\textbf{Exp1: Logarithmic Luenberger Regret.}
We consider the double integrator
$A \!=\! \left[\begin{smallmatrix}1 & 1\\0 & 1\end{smallmatrix}\right]$ ($r\!=\!2$), $C \!=\! [1,\,0]$,
with the noise model in \cref{eq:noise-model} using random-dominant amplitudes $C_w \!=\! C_v \!=\! 0.3$ and small bias and sinusoidal components. 
The fixed bias drives $\|y_t\|$ at the worst-case polynomial rate $\mathcal{O}(t^2)$ targeted by our analysis (\cref{fig:exp1}, left).
We compare standard OLS, FM-POLS with the Cayley--Hamilton (CH) hint $q_{\mathrm{ch}}(z) = (z-1)^2$, and FM-POLS with two Luenberger hints at $\tilde\gamma \in \{0.3, 0.8\}$, using $H = 15$, $\lambda = 1$, $T = 2000$.
\cref{fig:exp1}~(middle) shows the Luenberger regret growing linearly on the log time axis, consistent with $\mathcal{O}(H \log T)$ scaling for fixed $H$ (the $\mathcal{O}(\log^2 T)$ bound in \cref{thm:total-regret} arises when $H=\mathcal{O}(\log T)$).
The aggressive Luenberger hint ($\tilde\gamma = 0.8$) attains lower regret than the conservative one ($\tilde\gamma = 0.3$), reflecting the $\Delta_{\max}^2$ dependence in \cref{thm:total-regret};
\cref{fig:exp1}~(right) confirms that the aggressive gain attains smaller hint residual $\Delta_{\max}$.
The CH hint, which uses only the eigenvalue $\lambda = 1$, achieves comparable performance with Luenberger hints.

\textbf{Exp2: FM-POLS versus Classical Fixed-Gain Filters.}
We consider the symmetric system $A = [\begin{smallmatrix}0&1\\1&0\end{smallmatrix}]$ ($r\!=\!1$, with eigenvalues $\pm 1$),
$C = [1,\;0.5]$, under the noise model in \cref{eq:noise-model} with the bias-sine-dominant amplitudes ($\bar w \!=\! a_w \!=\! [0.1, 0.1]^\tr$ and $\bar v \!=\! a_v \!=\! 0.1$, while $C_w \!=\! C_v \!=\! 0.01$).
We compare two fixed-gain baselines, Kalman and $\mathcal{H}_\infty$ (designed via RMS-matched covariance), against FM-POLS with a model-based Luenberger hint and with the model-free 2-lag hint $\tilde{y}_t = y_{t-2}$, using
$H = 8$, $\lambda = 1$.
\cref{fig:exp2}~(middle) shows that Kalman and $\mathcal{H}_\infty$ incur large cumulative prediction loss that grows linearly, meaning their fixed gains fail to track the structured disturbance.
Both FM-POLS variants achieve much smaller cumulative loss. Remarkably, the 2-lag hint, which requires no knowledge of $(A, C)$, outperforms the model-based hint. \cref{fig:exp2}~(right) shows that FM-POLS per-step error decays to a small noise floor, while the fixed-gain filters maintain a persistent error.

\begin{wrapfigure}[14]{r}{0.58\linewidth}
    \vspace{-5.5mm}
    \centering
    \includegraphics[width=0.49\linewidth]{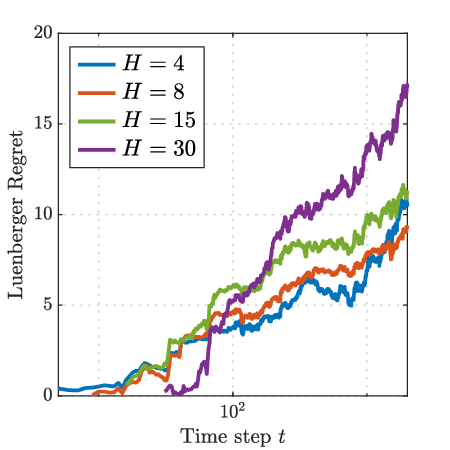} \hfill     \includegraphics[width=0.49\linewidth]{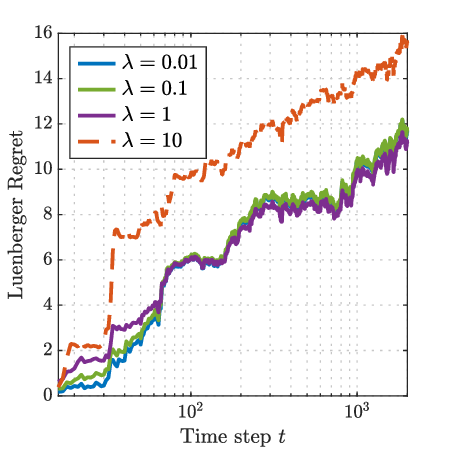}
    \caption{Exp3: Sensitivity of $H$ and $\lambda$. Left: Varying $H$ affects the regret constant. Right: Insensitivity to $\lambda$ within a moderate range; $\lambda = 10$ is too large, causing the bias term to dominate.}
    \label{fig:exp3-H}
\end{wrapfigure}

\textbf{Exp3: Sensitivity to $H$ and $\lambda$.}
Using the setup of Exp1 with the aggressive Luenberger hint, we first
vary the memory length $H$ with $\lambda = 1$ fixed.
\cref{fig:exp3-H} (left) shows that all values of $H$ produce linear curves on the log time axis, confirming $\mathcal{O}(H\log T)$ regret for fixed $H$; a larger $H$
results in a larger slope. 
This matches \cref{thm:total-regret} and reflects the tradeoff between approximation (truncation) and estimation (regression) error.
Next, we vary the regularization $\lambda$ with $H = 15$ fixed.
\cref{fig:exp3-H} (right) shows that regret is nearly insensitive to
$\lambda$ in $\{0.01, 0.1, 1\}$: the log-determinant term 
$\log(T^{2r+1}/\lambda)$ dominates and is insensitive to $\lambda$, while the subdominant bias term $\lambda \|M_L\|_F^2$ stays small.
For $\lambda = 10$, the bias exceeds
the log saving and causes visible degradation.
In contrast to standard OLS, which requires horizon-dependent tuning for adversarial state prediction~\cite{ghai2020no}, any fixed $\lambda = \mathcal{O}(1)$ suffices for FM-POLS.

\vspace{-3mm}
\section{Conclusion}
\vspace{-2mm}
\label{sec:conclusion}
We studied online nonstochastic prediction in marginally stable PO-LDS, and showed that polylogarithmic regret against the best-in-hindsight stabilizing Luenberger predictor is achievable.
Our framework FM-POLS uses tailored predictive hints to absorb the system's deterministic drift, so that regret scales with the hint residual rather than the unbounded signal. 
We further developed a model-based Luenberger hint that works for any stabilizing gain, and model-free polynomial-filter hints that suffice for real-spectrum systems.
For future directions, it is interesting to allow the hint residual to grow polylogarithmically with $T$, which may potentially broaden the class of admissible filters.
It is also interesting to extend FM-POLS to time-varying and structured nonlinear systems.



\newpage

\bibliographystyle{unsrt}
\bibliography{ref.bib}

@article{qian2025logarithmic,
  title={Logarithmic regret and polynomial scaling in online multi-step-ahead prediction},
  author={Qian, Jiachen and Zheng, Yang},
  journal={IEEE Control Systems Letters},
  volume={9},
  pages={2981--2986},
  year={2025},
  publisher={IEEE}
}

@ARTICLE{Luenberger1964Observing,
  author={Luenberger, David G.},
  journal={IEEE Transactions on Military Electronics}, 
  title={Observing the State of a Linear System}, 
  year={1964},
  volume={8},
  number={2},
  pages={74-80},
  doi={10.1109/TME.1964.4323124}}

@inproceedings{foster2020logarithmic,
  title={Logarithmic regret for adversarial online control},
  author={Foster, Dylan and Simchowitz, Max},
  booktitle={International Conference on Machine Learning},
  pages={3211--3221},
  year={2020},
  organization={PMLR}
}

@inproceedings{agarwal2019online,
  title={Online control with adversarial disturbances},
  author={Agarwal, Naman and Bullins, Brian and Hazan, Elad and Kakade, Sham and Singh, Karan},
  booktitle={International Conference on Machine Learning},
  pages={111--119},
  year={2019},
  organization={PMLR}
}

@article{agarwal2019logarithmic,
  title={Logarithmic regret for online control},
  author={Agarwal, Naman and Hazan, Elad and Singh, Karan},
  journal={Advances in Neural Information Processing Systems},
  volume={32},
  year={2019}
}

@inproceedings{simchowitz2020improper,
  title={Improper learning for non-stochastic control},
  author={Simchowitz, Max and Singh, Karan and Hazan, Elad},
  booktitle={Conference on Learning Theory},
  pages={3320--3436},
  year={2020},
  organization={PMLR}
}

@inproceedings{simchowitz2020naive,
  title={Naive exploration is optimal for online lqr},
  author={Simchowitz, Max and Foster, Dylan},
  booktitle={International Conference on Machine Learning},
  pages={8937--8948},
  year={2020},
  organization={PMLR}
}

@inproceedings{cassel2020logarithmic,
  title={Logarithmic regret for learning linear quadratic regulators efficiently},
  author={Cassel, Asaf and Cohen, Alon and Koren, Tomer},
  booktitle={International Conference on Machine Learning},
  pages={1328--1337},
  year={2020},
  organization={PMLR}
}

@inproceedings{cohen2018online,
  title={Online linear quadratic control},
  author={Cohen, Alon and Hasidim, Avinatan and Koren, Tomer and Lazic, Nevena and Mansour, Yishay and Talwar, Kunal},
  booktitle={International Conference on Machine Learning},
  pages={1029--1038},
  year={2018},
  organization={PMLR}
}

@article{hazan2017learning,
  title={Learning linear dynamical systems via spectral filtering},
  author={Hazan, Elad and Singh, Karan and Zhang, Cyril},
  journal={Advances in Neural Information Processing Systems},
  volume={30},
  year={2017}
}

@article{hazan2022introduction,
  title={Introduction to online nonstochastic control},
  author={Hazan, Elad and Singh, Karan},
  journal={arXiv preprint arXiv:2211.09619},
  year={2022}
}

@book{cesa2006prediction,
  title={Prediction, learning, and games},
  author={Cesa-Bianchi, Nicolo and Lugosi, G{\'a}bor},
  year={2006},
  publisher={Cambridge University Press}
}

@inproceedings{ghai2020no,
  title={No-regret prediction in marginally stable systems},
  author={Ghai, Udaya and Lee, Holden and Singh, Karan and Zhang, Cyril and Zhang, Yi},
  booktitle={Conference on Learning Theory},
  pages={1714--1757},
  year={2020},
  organization={PMLR}
}

@article{jacobsen2024online,
  title={Online linear regression in dynamic environments via discounting},
  author={Jacobsen, Andrew and Cutkosky, Ashok},
  journal={arXiv preprint arXiv:2405.19175},
  year={2024}
}

@article{orabona2015generalized,
  title={A generalized online mirror descent with applications to classification and regression},
  author={Orabona, Francesco and Crammer, Koby and Cesa-Bianchi, Nicolo},
  journal={Machine Learning},
  volume={99},
  number={3},
  pages={411--435},
  year={2015},
  publisher={Springer}
}

@inproceedings{kozdoba2019line,
  title={On-line learning of linear dynamical systems: Exponential forgetting in {K}alman filters},
  author={Kozdoba, Mark and Marecek, Jakub and Tchrakian, Tigran and Mannor, Shie},
  booktitle={Proceedings of the AAAI Conference on Artificial Intelligence},
  volume={33},
  number={01},
  pages={4098--4105},
  year={2019}
}

@article{azoury2001relative,
  title={Relative loss bounds for on-line density estimation with the exponential family of distributions},
  author={Azoury, Katy S and Warmuth, Manfred K},
  journal={Machine learning},
  volume={43},
  number={3},
  pages={211--246},
  year={2001},
  publisher={Springer}
}

@article{rashidinejad2020slip,
  title={{SLIP}: Learning to predict in unknown dynamical systems with long-term memory},
  author={Rashidinejad, Paria and Jiao, Jiantao and Russell, Stuart},
  journal={Advances in Neural Information Processing Systems},
  volume={33},
  pages={5716--5728},
  year={2020}
}

@inproceedings{simchowitz2019learning,
  title={Learning linear dynamical systems with semi-parametric least squares},
  author={Simchowitz, Max and Boczar, Ross and Recht, Benjamin},
  booktitle={Conference on Learning Theory},
  pages={2714--2802},
  year={2019},
  organization={PMLR}
}

@article{hazan2018spectral,
  title={Spectral filtering for general linear dynamical systems},
  author={Hazan, Elad and Lee, Holden and Singh, Karan and Zhang, Cyril and Zhang, Yi},
  journal={Advances in Neural Information Processing Systems},
  volume={31},
  year={2018}
}

@article{marsden2025universal,
  title={Universal sequence preconditioning},
  author={Marsden, Annie and Hazan, Elad},
  journal={arXiv preprint arXiv:2502.06545},
  year={2025}
}

@article{tsiamis2022online,
  title={Online learning of the {K}alman filter with logarithmic regret},
  author={Tsiamis, Anastasios and Pappas, George J},
  journal={IEEE Transactions on Automatic Control},
  volume={68},
  number={5},
  pages={2774--2789},
  year={2022},
  publisher={IEEE}
}

@inproceedings{rakhlin2013online,
  title={Online learning with predictable sequences},
  author={Rakhlin, Alexander and Sridharan, Karthik},
  booktitle={Conference on Learning Theory},
  pages={993--1019},
  year={2013},
  organization={PMLR}
}

@article{kalman1960new,
  title={A new approach to linear filtering and prediction problems},
  author={Kalman, Rudolph Emil},
  journal={Transactions of the ASME--Journal of Basic Engineering},
  year={1960}
}

@book{simon2006optimal,
  title={Optimal state estimation: {K}alman, H-infinity, and nonlinear approaches},
  author={Simon, Dan},
  year={2006},
  publisher={John Wiley \& Sons}
}

@article{orabona2019modern,
  title={A modern introduction to online learning},
  author={Orabona, Francesco},
  journal={arXiv preprint arXiv:1912.13213},
  year={2019}
}

@article{vovk2001competitive,
  title={Competitive on-line statistics},
  author={Vovk, Volodya},
  journal={International Statistical Review},
  volume={69},
  number={2},
  pages={213--248},
  year={2001},
  publisher={Wiley Online Library}
}

@article{qian2025model,
  title={Model-free online learning for the {K}alman filter: Forgetting factor and logarithmic regret},
  author={Qian, Jiachen and Zheng, Yang},
  journal={arXiv preprint arXiv:2505.08982},
  year={2025}
}

@article{rakhlin2013optimization,
  title={Optimization, learning, and games with predictable sequences},
  author={Rakhlin, Sasha and Sridharan, Karthik},
  journal={Advances in Neural Information Processing Systems},
  volume={26},
  year={2013}
}

@article{hazan2007logarithmic,
  title={Logarithmic regret algorithms for online convex optimization},
  author={Hazan, Elad and Agarwal, Amit and Kale, Satyen},
  journal={Machine Learning},
  volume={69},
  number={2},
  pages={169--192},
  year={2007},
  publisher={Springer}
}

@article{hazan2016introduction,
  title={Introduction to online convex optimization},
  author={Hazan, Elad},
  journal={Foundations and Trends in Optimization},
  volume={2},
  number={3-4},
  pages={157--325},
  year={2016},
  publisher={Emerald Publishing Limited}
}

@book{box2015time,
  title={Time series analysis: forecasting and control},
  author={Box, George EP and Jenkins, Gwilym M and Reinsel, Gregory C and Ljung, Greta M},
  year={2015},
  publisher={John Wiley \& Sons}
}


\newpage
\tableofcontents

\newpage
\appendix

\section*{Appendix}

The appendix is organized as follows:
\begin{itemize}
    \item Appendix~\ref{app:section2}: omitted details for \cref{sec:problem_setup} (polynomial trajectory growth).
    \item Appendix~\ref{app:section3}: omitted details for \cref{sec:pols_framework} (proofs for the POLS framework, finite-memory truncation analysis, and Luenberger regret of FM-POLS).
    \item Appendix~\ref{app:section4}: omitted details for \cref{sec:hint_design} (proofs for the model-free hint results).
    \item Appendix~\ref{app:additional-exp}: additional numerical experiments.
\end{itemize}
A summary of paper-specific notation is provided in \cref{tab:notation} below.

\begin{table}[h]
\centering
\caption{Notation used throughout the paper.}
\label{tab:notation}
\begin{tabular}{ll}
\toprule
\textbf{Object} & \textbf{Notation} \\
\midrule
\multicolumn{2}{l}{\emph{System and observations}} \\
True observation & $y_t$ \\
True latent state & $x_t$ \\
Process and measurement noise & $w_t,\ v_t$ \\
Disturbance bounds & $\|w_t\| \le C_w$,\ $\|v_t\| \le C_v$ \\
System matrices & $A,\ C$ \\
Observation/state dimension & $p,\ n$ \\
Polynomial-growth bound on $\|y_t\|$ & $C_y$ (\cref{lem:poly-growth}) \\
Jordan-block size bound & $r$ \\
Eigenbasis conditioning & $\kappa_A := \|S\|\|S^{-1}\|$ \\
\midrule
\multicolumn{2}{l}{\emph{Comparator class}} \\
Luenberger gain & $L$ \\
Closed-loop matrix & $A_L := A - LC$ \\
Strong-stability parameters & $(\kappa, \gamma)$ \\
Comparator class & $\mathcal{L}_{\kappa,\gamma}$ \\
Best hindsight gain & $L^\star$ \\
Comparator's prediction & $\hat{y}^{\mathrm{LB}}_t(L)$ \\
Comparator's internal state & $\hat{x}^{\mathrm{LB}}_t(L)$ \\
\midrule
\multicolumn{2}{l}{\emph{Algorithm: FM-POLS}} \\
Learner's prediction & $\hat{y}_t$ \\
Predictive hint & $\tilde{y}_t$ \\
Hint residual & $\Delta_t := y_t - \tilde{y}_t$ \\
Maximum hint residual & $\Delta_{\max} := \max_{t \in [T]} \|\Delta_t\|$ \\
Hint gain (model-based) & $\tilde{L}$ \\
Strong-stability parameters of $\tilde{L}$ & $(\tilde{\kappa}, \tilde{\gamma})$ \\
Hint's internal state & $\hat{x}_t^{\mathrm{LB}}(\tilde{L})$ \\
POLS predictor matrix & $M_t$ \\
POLS feature & $z_t\in\mathbb{R}^d$ \\
FM-POLS feature (history of past outputs) & $z_t = [y_{t-1}^\tr, \ldots, y_{t-H}^\tr]^\tr \in\mathbb{R}^{pH}$ \\
Memory length (truncation depth) & $H$ \\
Regularization parameter & $\lambda$ \\
\midrule
\multicolumn{2}{l}{\emph{Truncation and analysis}} \\
Truncated comparator prediction & $\hat{y}^{\mathrm{TR}}_t(L) = M_L z_t$ \\
Comparator matrix (in $M$-space) & $M_L := [CL,\ CA_L L,\ \ldots,\ CA_L^{H-1}L]$ \\
Truncation error constant & $C_{\mathrm{trun}}$ (\cref{lem:luen-trunc}) \\
Prediction error constant & $C_{\mathrm{pred}}$ (\cref{lem:luen-pred}) \\
\midrule
\multicolumn{2}{l}{\emph{Polynomial filter (model-free hints)}} \\
Filter polynomial & $q(z) = \sum_{i=0}^m c_i z^{m-i}$ \\
Filter coefficient $\ell_1$-norm & $\|q\|_1 = \sum_{i} |c_i|$ \\
\bottomrule
\end{tabular}
\end{table}

\newpage

\section{Omitted Details for \cref{sec:problem_setup}}
\label{app:section2}

\subsection{Proof of \cref{lem:poly-growth} (Polynomial Growth)}
\label{app:poly-growth}

\begin{proof}
Part (i) bounds $\|A^k\|$ by reducing to a single Jordan block and using binomial expansions.
Part (ii) and (iii) are straightforward from Part (i).

\textbf{(i) Power bound.} 
Since $A=SJS^{-1}$, we have $\|A^k\|=\|SJ^kS^{-1}\|\le \kappa_A\,\|J^k\|$.
Because $J$ is block-diagonal, $\|J^k\|$ equals the maximum operator norm across its Jordan blocks, 
\[
    \|J^k\|=\max_{i}\|J_{\lambda_i}^k\|,
\]
so it suffices to bound a single Jordan block of size at most $r$ (by \cref{assum:main}-(ii)).

Consider any such Jordan block $J_{\lambda}=\lambda I + N$ with $|\lambda| \le 1$, where $N$ is the nilpotent shift matrix. 
Since the block has size at most $r$, we have $N^\ell = 0$ for $\ell \ge r$, so the binomial expansion gives
\[
    J_{\lambda}^k = (\lambda I + N)^k = \sum_{\ell=0}^{\min\{k,r-1\}} \binom{k}{\ell} \lambda^{k-\ell} N^\ell.
\]
The nilpotent shift matrix satisfies $\|N^\ell\|\le 1$ for all $\ell\ge0$, so
\begin{equation}
    \label{eq:Jordan-power-bound}
  \|J_{\lambda}^{k}\|
  \le \sum_{\ell=0}^{\min\{k,r-1\}}\binom{k}{\ell}|\lambda|^{k-\ell}
  \le \sum_{\ell=0}^{\min\{k,r-1\}}\binom{k}{\ell}
  \le \sum_{\ell=0}^{\min\{k,r-1\}}k^{\ell}
  \le (1+k)^{r-1},
\end{equation}
where the second inequality uses $|\lambda|\le 1$ and $k\ge \ell$;
the third uses $\binom{k}{\ell}=\frac{k(k-1)\cdots(k-\ell+1)}{\ell!}\le k^\ell$;
the last uses the binomial expansion below
\[
    (1+k)^{r-1}=\sum_{\ell=0}^{r-1}\binom{r-1}{\ell}k^\ell\ge \sum_{\ell=0}^{r-1}k^\ell \ge \sum_{\ell=0}^{\min\{k,r-1\}}k^{\ell}
\]
since $\binom{r-1}{\ell}\ge 1$.
Since the bound in \cref{eq:Jordan-power-bound} holds uniformly for each Jordan block of $J$, we conclude
\[
    \|A^k\|\le \kappa_A \|J^k\| \le \kappa_A \max_{i}\|J_{\lambda_i}^k\| \le \kappa_A (1+k)^{r-1}.
\]

\textbf{(ii) Cumulative power bound.}
For each $0\le s\le k$, the bound from (i) gives $\|A^s\|\le\kappa_A(1+s)^{r-1}\le \kappa_A(1+k)^{r-1}$.
Summing the $k+1$ such terms gives
\begin{equation*}
  \sum_{s=0}^{k}\|A^s\|
  \le\kappa_A\,(k+1)\,(1+k)^{r-1}
  =\kappa_A\,(1+k)^{r}.
\end{equation*}

\textbf{(iii) Trajectory bound.}
Unrolling the state dynamics $x_{t+1}=Ax_t+w_t$ from $x_0=0$ gives 
$
x_t=\sum_{s=0}^{t-1}A^s w_{t-1-s}.
$
Using $\|w_s\|\le C_w$ and~(ii) with $k=t-1$,
\begin{equation*}
  \|x_t\|
  \le C_w\sum_{s=0}^{t-1}\|A^s\|
  \le \underbrace{\kappa_A C_w}_{=: C_x}(1+t)^{r}.
\end{equation*}
Then
\begin{equation*}
  \|y_t\|
  \le \|C\|\|x_t\|+\|v_t\|
  \le\|C\|C_x(1+t)^{r}+C_v
  \le\underbrace{\left(\|C\|C_x+C_v\right)}_{:= C_y}(1+t)^{r},
\end{equation*}
where the last inequality uses $C_v\le C_v(1+t)^r$ for $t\ge 1$.
\end{proof}

\section{Omitted Details for \cref{sec:pols_framework}}
\label{app:section3}

\subsection{Proof of \cref{thm:pols-reg} (Residual-Scaled POLS Regret)}
\label{app:reg-analysis}

Our proof relies on the framework of \cite{orabona2015generalized} and is conducted in two steps: a scalar POLS regret bound (Appendix~\ref{app:scalar-pols}) followed by a direct extension to the matrix case (Appendix~\ref{app:matrix-pols}).
We then show that standard OLS is as a special case of POLS (Appendix~\ref{app:ols-as-pols}).
We begin with two useful lemmas.

\begin{lemma}
\label{lem:rank1-det}
For any $v \in \mathbb{R}^d$,\;
$\det(I - vv^\tr) = 1 - \|v\|^2$.
\end{lemma}

\vspace{-5mm}
\begin{proof}
The vector $v$ is an eigenvector of $I - vv^\tr$
with eigenvalue $1 - \|v\|^2$, and every
$u \perp v$ is an eigenvector with eigenvalue $1$.
The determinant is the product of all $d$
eigenvalues.
\end{proof}

\begin{lemma}[Elliptical Potential]
\label{lem:ell-potential}
Let $z_t \in \mathbb{R}^{d}$, $G_0 \!:=\! \lambda I$, and
$G_t \!:=\! G_{t-1} + z_t z_t^\tr$ for $t\ge 1$. Then
\begin{equation*}
    \sum_{t=1}^T z_t^\tr G_t^{-1} z_t
    \le \log \frac{\det G_T}{\det G_0}
    \le d \log\!\left(
      1 + \frac{\sum_{t=1}^T \|z_t\|^2}{\lambda d}
    \right).
\end{equation*}
\end{lemma}

\begin{proof}
Since
$G_{t-1} = G_t - z_t z_t^\tr
= G_t^{\frac{1}{2}}(I - G_t^{-\frac{1}{2}} z_t z_t^\tr G_t^{-\frac{1}{2}})
G_t^{\frac{1}{2}}$,
taking determinants and applying
\cref{lem:rank1-det} with
$v = G_t^{-\frac{1}{2}} z_t$ gives
$
    {\det G_{t-1}}/{\det G_t}
    = 1 - \|G_t^{-\frac{1}{2}} z_t\|^2
    = 1 - z_t^\tr G_t^{-1} z_t.
$
Rearranging and using
$1 - x \le -\log x$ for $0<x\le 1$ yields
\[
    z_t^\tr G_t^{-1} z_t
    = 1 - \frac{\det G_{t-1}}{\det G_t}
    \le \log \frac{\det G_t}{\det G_{t-1}}.
\]
Telescoping then proves the first inequality.
Now we show $z_t^\tr G_t^{-1} z_t < 1$.
Since
$G_t - z_t z_t^\tr \succ 0$,
equivalent to
$I - G_t^{-\frac{1}{2}} z_t z_t^\tr G_t^{-\frac{1}{2}} \succ 0$.
The only nonzero eigenvalue of the rank-one matrix
$G_t^{-\frac{1}{2}} z_t z_t^\tr G_t^{-\frac{1}{2}}$ is
$z_t^\tr G_t^{-1} z_t$, so it must be strictly less than $1$.

For the second inequality, note that
$
{\det G_T}/{\det G_0} 
= \det\left(I + \lambda^{-1}\sum_t z_t z_t^\tr\right).
$
By the AM--GM inequality applied to the eigenvalues
$\mu_1, \ldots, \mu_{d}$ of
$I + \lambda^{-1}\sum_t z_t z_t^\tr$,
\[
    \prod_{i=1}^{d} \mu_i
    \le \left(\frac{\sum_i \mu_i}{d}\right)^{d}
    = \left(
      1 + \frac{\operatorname{tr}(\sum_t z_t z_t^\tr)}
               {\lambda d}
    \right)^{d}
    = \left(
      1 + \frac{\sum_t \|z_t\|^2}{\lambda d}
    \right)^{d}.
\]
Taking logarithms on both sides completes the proof.
\end{proof}

\subsubsection{Scalar POLS Regret}
\label{app:scalar-pols}

Consider the scalar case ($p = 1$). The learner predicts $\hat{y}_t = m_t^\tr z_t$ via the POLS update
\[
    m_t
    = \argmin_{m \in \mathbb{R}^{d}}
    \left\{
      \lambda\|m\|^2
      + \sum_{s=1}^{t-1}(m^\tr z_s - y_s)^2
      + (m^\tr z_t - \tilde{y}_t)^2
    \right\},
\]
with closed form
$m_t = G_t^{-1}(b_{t-1} + z_t\tilde{y}_t)$
where $b_{t-1} := \sum_{s=1}^{t-1} z_s y_s$ and
$G_t := \lambda I + \sum_{s=1}^t z_s z_s^\tr$.

We prove the following intermediate bound, which is sharper than the ``$\max_t$'' form of \cref{thm:pols-reg} and is needed for our matrix extension in Appendix~\ref{app:matrix-pols} without incurring an extra factor of $p$.

\begin{lemma}[Scalar POLS Regret]
\label{lem:scalar-pols}
For any comparator $u \in \mathbb{R}^{d}$ and any sequences
$(z_t, y_t, \tilde{y}_t)_{t=1}^T$,
\begin{equation*}
    \sum_{t=1}^T \bigl[
      (m_t^\tr z_t - y_t)^2
      - (u^\tr z_t - y_t)^2
    \bigr]
    \le \lambda\|u\|^2
    + \sum_{t=1}^T \Delta_t^2\,
      z_t^\tr G_t^{-1} z_t,
\end{equation*}
where $\Delta_t := y_t - \tilde{y}_t$ is the hint residual.
\end{lemma}

\begin{proof}

Our proof follows the FTRL-with-time-varying-regularizer framework of \cite{orabona2015generalized}. The key step tailored to our setting is a \emph{tilted} reparameterization that absorbs the hint-related look-ahead term $(m^\tr z_t-\tilde y_t)^2$ into the regularizer, leaving a linear loss whose gradient depends only on the hint residual~$\Delta_t=y_t-\tilde y_t$.

\textbf{Squared loss decomposition.}
Define the \textit{tilted} linear loss and {time-varying} regularizer for $t\ge 1$
\[
\tilde\ell_t(m) := -2\Delta_t\, m^\tr z_t, \quad R_t(m) := m^\tr G_t\, m - 2c_t^\tr m, \quad c_t := \sum_{s=1}^t z_s\tilde{y}_s,
\]
with the convention $R_0(m) := \lambda\|m\|^2.$
Using $G_t-G_{t-1}=z_t z_t^\tr$ and $c_t-c_{t-1}=z_t\tilde y_t$, the regularizer increment satisfies
$
  R_t(m) - R_{t-1}(m) = (m^\tr z_t)^2 - 2\tilde y_t m^\tr z_t.
$
A direct computation using $y_t=\tilde y_t+\Delta_t$ then yields the central identity
\begin{equation}\label{eq:loss-decomp}
  (m^\tr z_t-y_t)^2
  \;=\;\underbrace{[R_t(m)-R_{t-1}(m)]}_{\mathtt{regularizer~increment}}
  \;+\;\underbrace{\tilde\ell_t(m)}_{\mathtt{linear~loss}}
  \;+\;y_t^2.
\end{equation}
The same algebra gives
\begin{equation}\label{eq:lookahead-decomp}
  (m^\tr z_t-\tilde y_t)^2 \;=\; [R_t(m)-R_{t-1}(m)] + \tilde y_t^2.
\end{equation}

\textbf{POLS is FTRL on linear losses.}
With the above decomposition, we claim that $m_t$ coincides with the FTRL iterate on $(\tilde\ell_s)_{s<t}$ with the time-varying regularizer $R_t$
\begin{equation}\label{eq:pols-is-ftrl}
  m_t = \argmin_{m\in\mathbb{R}^d}\left\{R_t(m) + \sum_{s=1}^{t-1}\tilde\ell_s(m)\right\}.
\end{equation}
This can be shown by substituting \cref{eq:loss-decomp} and \cref{eq:lookahead-decomp} into the POLS objective
\begin{align*}
    & \underbrace{R_0(m)}_{=\lambda\|m\|^2}
    + \sum_{s=1}^{t-1}
      \underbrace{\bigl[R_s(m) - R_{s-1}(m)
      + \tilde\ell_s(m) + y_s^2\bigr]}_{=(m^\tr z_t-y_t)^2}
    + \underbrace{[R_t(m) - R_{t-1}(m)] + \tilde{y}_t^2}_{=(m^\tr z_t - \tilde{y}_t)^2}  \\
    = & R_t(m) + \sum_{s=1}^{t-1}\tilde\ell_s(m) + 
    \underbrace{\sum_{s=1}^{t-1}y_s^2 + \tilde{y}_t^2}_{\mathtt{const}},
\end{align*}
where the equality uses
$R_0(m) + \sum_{s=1}^{t-1}[R_s(m) - R_{s-1}(m)]
= R_{t-1}(m)$ and the constant term collects terms that do not depend on $m$.

\textbf{FTRL regret bound.}
The regularizer $R_t$ has Hessian $\nabla^2 R_t = 2G_t$, so it is $\beta$-strongly convex with respect to $\|\cdot\|_{G_t}$ with $\beta=2$. The standard FTRL-with-time-varying-regularizer bound in \cite{orabona2015generalized} gives, for any comparator $u$,
\begin{equation}\label{eq:ftrl-bound}
  \sum_{t=1}^T\bigl[\tilde\ell_t(m_t)-\tilde\ell_t(u)\bigr]
  \le R_T(u) + \sum_{t=1}^T\Bigl(\frac{1}{2\beta}\|g_t\|_{G_t^{-1}}^2 + R_{t-1}(m_t)-R_t(m_t)\Bigr),
\end{equation}
where $g_t=\nabla\tilde\ell_t(m_t) = -2 z_t \Delta_t$ and the prefactor $\tfrac{1}{2\beta}=\tfrac{1}{4}$.

\textbf{Assembling the regret.}
We decompose the regret $\Reg_T(m)=\sum_{t=1}^T\bigl[(m_t^\tr z_t-y_t)^2 - (u^\tr z_t-y_t)^2\bigr]$ using the loss decomposition \cref{eq:loss-decomp}
\[
    \Reg_T(m)= \underbrace{\sum_{t=1}^T[R_t(m_t)-R_{t-1}(m_t)]}_{\term{a}}
    - \underbrace{\sum_{t=1}^T [R_t(u) - R_{t-1}(u)]}_{=R_T(u)-R_0(u)}
    + \underbrace{\sum_{t=1}^T[\tilde\ell_t(m_t)-\tilde\ell_t(u)]}_{\term{b}}.
\]
Bounding $\term{b}$ by~\eqref{eq:ftrl-bound} and collecting terms, we obtain
\begin{equation*}
  \Reg_T(m)
  \le
  \underbrace{\term{a} + \sum_{t=1}^T[R_{t-1}(m_t)-R_t(m_t)]}_{=0}
  +\underbrace{R_T(u)-[R_T(u)-R_0(u)]}_{=R_0(u)}
  +\tfrac{1}{4}\sum_{t=1}^T\|g_t\|_{G_t^{-1}}^2.
\end{equation*}
Substituting $R_0(u)=\lambda\|u\|^2$ and $\|g_t\|_{G_t^{-1}}^2=g_t^\tr G_t^{-1} g_t=4\Delta_t^2\,z_t^\tr G_t^{-1}z_t$ establishes \cref{lem:scalar-pols}.
\end{proof}

\begin{remark}
\label{rem:jacobsen}
The bound in \cref{lem:scalar-pols}
can also be obtained from
\cite[Theorem~3.1]{jacobsen2024online} by setting the
discount factor $\gamma = 1$ and extracting the bound
before the $\max_t$ step. The FTRL argument above is self-contained and makes the derivation of this intermediate bound explicit.
\end{remark}

\subsubsection{Extension to the Matrix Case}
\label{app:matrix-pols}

We now extend \cref{lem:scalar-pols} to the matrix case where $\hat{y}_t = M_t z_t$ with $M_t \in \mathbb{R}^{p \times d}$ ($p>1$). 

\begin{proof}[Proof of \cref{thm:pols-reg}]
We show that the matrix POLS objective in \cref{eq:pols-objective} decouples across the $p$ output coordinates, allowing us to apply the scalar regret bound of Lemma~\ref{lem:scalar-pols} row by row.

\textbf{Row-wise objective decoupling.}
Write the comparator in its row form
$M = [m_1^\tr, \ldots, m_{p}^\tr]^\tr$, and similarly for the POLS iterate $M_t^{\mathrm{pols}} = [m_{t,1}^\tr, \ldots, m_{t,p}^\tr]^\tr$.
Observe that the POLS objective in \cref{eq:pols-objective} splits as a sum of $p$ independent scalar POLS objectives, one per coordinate:
\[
    \lambda\|M\|_F^2
    + \sum_{s=1}^{t-1}\|Mz_s - y_s\|^2
    + \|Mz_t - \tilde{y}_t\|^2
    = \sum_{j=1}^{p} \Bigl[
      \lambda\|m_j\|^2
      + \sum_{s=1}^{t-1}(m_j^\tr z_s - y_{s,j})^2
      + (m_j^\tr z_t - \tilde{y}_{t,j})^2
    \Bigr].
\]
Therefore, the minimizer over $M$ can be obtained by minimizing each coordinate separately, so each row $m_{t,j}$ of $M_t^{\mathrm{pols}}$ is the scalar POLS iterate for coordinate $j$ (with target $y_{s,j}$ and hint $\tilde y_{t,j}$).
Crucially, all $p$ coordinates share the same Gram matrix $G_t=\lambda I+\sum_{s\le t}z_s z_s^\tr$.

\textbf{Row-wise regret decomposition.}
The matrix regret similarly decomposes coordinatewise: $\|M_t z_t - y_t\|^2 = \sum_{j=1}^p (m_{t,j}^\tr z_t - y_{t,j})^2$, and likewise for the comparator. Hence
\[
    \Reg_T(M)
    = \sum_{t=1}^T \bigl[\|M_t z_t - y_t\|^2 - \|Mz_t - y_t\|^2\bigr]
    = \sum_{j=1}^{p} \underbrace{\sum_{t=1}^T \bigl[(m_{t,j}^\tr z_t - y_{t,j})^2 - (m_j^\tr z_t - y_{t,j})^2\bigr]}_{=:\Reg_T^{(j)}(m_j)},
\]
where $\Reg_T^{(j)}(m_j)$ is the scalar regret for coordinate $j$ against the per-coordinate comparator $m_j$. 
Applying \cref{lem:scalar-pols} to each $\Reg_T^{(j)}(m_j)$ and summing over $j$ yields
\begin{align}
    \Reg_T(M)
    &\le \sum_{j=1}^{p} \Bigl[
      \lambda\|m_j\|^2
      + \sum_{t=1}^T
        (y_{t,j} - \tilde{y}_{t,j})^2\,
        z_t^\tr G_t^{-1} z_t
    \Bigr] \notag \\
    &= \lambda\|M\|_F^2
      + \sum_{t=1}^T
        \underbrace{\Bigl(\sum_{j=1}^{p}
          (y_{t,j} - \tilde{y}_{t,j})^2
        \Bigr)}_{= \|y_t - \tilde{y}_t\|^2}
        z_t^\tr G_t^{-1} z_t,
    \label{eq:matrix-intermediate}
\end{align}
where the last equality uses $\sum_j \|m_j\|^2 = \|M\|_F^2$ and combines the per-coordinate residuals via $\sum_j (y_{t,j} - \tilde{y}_{t,j})^2 = \|y_t - \tilde{y}_t\|^2$.
We note that, in the last equality, the order swap of the double sum is what allows the per-coordinate residuals $(y_{t,j} - \tilde{y}_{t,j})^2$ to combine into the vector residual $\|y_t-\tilde{y}_t\|^2$ \emph{before} the maximum over $t$ is extracted; performing the maximum first coordinate-wise would introduce an extra factor of $p$ (cf.~\cite[Appendix~F]{ghai2020no}).

\textbf{Concluding the bound.}
Extracting the uniform bound for $\|y_t-\tilde{y}_t\|^2$ from \cref{eq:matrix-intermediate} and applying the elliptical-potential bound from \cref{lem:ell-potential}, we obtain
\begin{align}
    \Reg_T(M)
    &\le \lambda\|M\|_F^2
      + \max_{1\le t\le T}\|y_t - \tilde{y}_t\|^2 
        d\log\!\left(\!
          1 + \frac{\sum_{t=1}^T \|z_t\|^2}
                   {\lambda d}
        \!\right),
    \notag
\end{align}
which is \cref{eq:pols_regret} and completes the proof.
\end{proof}

\subsubsection{OLS as a Special Case of POLS}
\label{app:ols-as-pols}

\begin{proposition}[Self-Consistent Hint Recovers OLS]
\label{prop:ols-as-pols}
Setting the self-consistent hint $\tilde{y}_t := M_t^{\mathrm{pols}} z_t$ in the POLS objective recovers the standard OLS iterates. That is,
\begin{equation*}
  M_t^{\mathrm{pols}} \;=\; M_t^{\mathrm{ols}} \;:=\; B_{t-1}\, G_{t-1}^{-1},
\end{equation*}
where $B_{t-1} := \sum_{s=1}^{t-1} y_s z_s^\tr
\in \mathbb{R}^{p \times d}$ and $G_t := \lambda I + \sum_{s=1}^t z_s z_s^\tr \in \mathbb{R}^{d \times d}$.
\end{proposition}

The self-consistent hint $\tilde{y}_t := M_t^{\mathrm{pols}} z_t$ is where the learner uses its own prediction as the hint.
\cref{prop:ols-as-pols} shows that this choice collapses POLS to OLS.

\begin{proof}
We first derive the closed form of the matrix POLS iterate. Setting the gradient of the POLS objective in \cref{eq:pols-objective} with respect to $M$ to zero gives
\begin{equation*}
  \lambda M + \sum_{s=1}^{t-1}(Mz_s-y_s)z_s^\tr + (Mz_t - \tilde y_t) z_t^\tr \;=\; 0,
\end{equation*}
which rearranges to $M_t^{\mathrm{pols}} G_t = B_{t-1} + \tilde y_t z_t^\tr$, i.e.,
\begin{equation}\label{eq:pols-closed-form}
  M_t^{\mathrm{pols}} \;=\; (B_{t-1} + \tilde y_t z_t^\tr)\, G_t^{-1}.
\end{equation}

When $\tilde y_t = M_t^{\mathrm{pols}} z_t$, \cref{eq:pols-closed-form} becomes a fixed-point equation in $M_t^{\mathrm{pols}}$:
\begin{equation*}
  M_t^{\mathrm{pols}} G_t \;=\; B_{t-1} + M_t^{\mathrm{pols}} z_t z_t^\tr.
\end{equation*}
Subtracting $M_t^{\mathrm{pols}} z_t z_t^\tr$ from both sides and using $G_t - z_t z_t^\tr = G_{t-1}$, we obtain
$
  M_t^{\mathrm{pols}} G_{t-1} \;=\; B_{t-1},
$
so $M_t^{\mathrm{pols}} = B_{t-1} G_{t-1}^{-1} = M_t^{\mathrm{ols}}$, as claimed.
\end{proof}

\begin{corollary}[Recovered OLS Regret Bound]\label{cor:ols-regret}
Under the self-consistent hint $\tilde y_t = M_t^{\mathrm{pols}} z_t$, the POLS regret bound \cref{eq:pols_regret} reduces to the classical OLS regret bound \cref{eq:ols_regret}: for any comparator $M\in\mathbb{R}^{p\times d}$,
\begin{align*}
  \mathrm{Reg}_T(M)
  &\le \lambda\|M\|_F^2 + \sum_{t=1}^T \|y_t - M_t^{\mathrm{ols}} z_t\|^2\, z_t^\tr G_t^{-1} z_t,\\
  &\le
    {\lambda}\|M\|_F^2 + \max_{1\le t\le T}\|y_t-M_t^{\mathrm{ols}} z_t\|^2 d \log\left(\!1+\frac{\sum_{t=1}^T\|z_t\|^2}{\lambda d}\!\right)
\end{align*}
\end{corollary}

\begin{proof}
By Proposition~\ref{prop:ols-as-pols}, the self-consistent hint gives $M_t^{\mathrm{pols}} = M_t^{\mathrm{ols}}$, so the hint residual becomes the OLS prediction error: $\Delta_t = y_t - \tilde y_t = y_t - M_t^{\mathrm{ols}} z_t$. 
Substituting into \cref{eq:pols_regret} and using \cref{lem:ell-potential} yields the stated bound \cref{eq:ols_regret}, which matches \cite[Thm~11.7]{cesa2006prediction} and \cite[Thm~1]{ghai2020no}.
\end{proof}

\subsection{Proof of \cref{lem:luen-trunc} (Luenberger Truncation Error)}
\label{app:trunc}

\begin{proof}
Unrolling \cref{eq:luen-dynamics} from $\hat x_0^{\mathrm{LB}}=0$ gives 
\[
\hat{y}_t^{\mathrm{LB}}(L)
= \sum_{k=0}^{t-1} CA_L^k L\,y_{t-k-1}.
\]
The truncated predictor $\hat{y}_t^{\mathrm{TR}}(L)$ retains only the first $H$ terms, so their difference is
\begin{equation}\label{eq:trunc-tail}
    \hat{y}_t^{\mathrm{LB}}(L) - \hat{y}_t^{\mathrm{TR}}(L)
    = \sum_{k=H}^{t-1} CA_L^k L\, y_{t-k-1},
\end{equation}
which vanishes for $t \le H$ so the bound holds trivially.
We focus on the case $t\ge H+1$.

\textbf{Pointwise bound.}
By strong stability of $A_L$ and \cref{lem:poly-growth}, we have
\[
    \|A_L^k\| \le \kappa(1-\gamma)^k, \quad \|L\| \le \kappa, \quad \|y_s\| \le C_y(1+s)^r.
\]
Therefore, with \cref{eq:trunc-tail},
\begin{align}
    \|\hat{y}_t^{\mathrm{LB}}(L) - \hat{y}_t^{\mathrm{TR}}(L)\|
    &\le \|C\|\kappa^2 \sum_{k=H}^{t-1}
         (1-\gamma)^k C_y(1+t-k-1)^r
         \notag \\
    &\le \|C\|\kappa^2 C_y(1+t)^r
         \sum_{k=H}^{\infty} (1-\gamma)^k
         \notag \\
   & = \frac{\|C\|\kappa^2 C_y}{\gamma}
      (1+t)^r(1-\gamma)^H.
      \label{eq:ptwise-detail}
\end{align}
The error grows at most polynomially in $t$ (with degree $r$), while the dependence on $H$ is geometric.

\textbf{Cumulative bounds.}
We choose $H$ so that the geometric factor $(1-\gamma)^H$ absorbs both the polynomial growth $(1+t)^r$ and the summation over $t \in [T]$.
Using $\log(1-\gamma) \le -\gamma$ for $\gamma\in(0,1]$, the choice $H \ge \lceil \gamma^{-1}(r+1)\log(1+T) \rceil$ ensures
\begin{equation}\label{eq:H-choice}
  (1-\gamma)^H \le e^{-\gamma H} \le (1+T)^{-(r+1)}.
\end{equation}


Summing \cref{eq:ptwise-detail} and using \cref{eq:H-choice} gives
\begin{align*}
    \sum_{t=1}^T \|\hat{y}_t^{\mathrm{LB}}(L) - \hat{y}_t^{\mathrm{TR}}(L)\|
    &\le \frac{\|C\| \kappa^2 C_y}{\gamma}\,
         (1-\gamma)^H\sum_{t=1}^T(1+t)^r \\
    &\le \frac{\|C\| \kappa^2 C_y}{\gamma}\,
         (1+T)^{-(r+1)}\cdot T(1+T)^r \\
    &= \frac{\|C\| \kappa^2 C_y}{\gamma}\,
       \frac{T}{1+T}
    \le \frac{\|C\| \kappa^2 C_y}{\gamma}
    =: C_{\mathrm{trun}}.
    \qedhere
\end{align*}

Summing the square of \cref{eq:ptwise-detail} and using \cref{eq:H-choice} yields
\begin{align*}
    \sum_{t=1}^T \|\hat{y}_t^{\mathrm{LB}}(L) - \hat{y}_t^{\mathrm{TR}}(L)\|^2
    &\le \frac{\|C\|^2 \kappa^4 C_y^2}{\gamma^2}\,
         (1-\gamma)^{2H}\sum_{t=1}^T(1+t)^{2r} \\
    &\le \frac{\|C\|^2 \kappa^4 C_y^2}{\gamma^2}\,
         (1+T)^{-2(r+1)}\cdot T(1+T)^{2r} \\
    &= \frac{\|C\|^2 \kappa^4 C_y^2}{\gamma^2}\,
       \frac{T}{(1+T)^2}
    \le \frac{\|C\|^2 \kappa^4 C_y^2}{\gamma^2}
    = C_{\mathrm{trun}}^2.
\end{align*}
\end{proof}

\subsection{Proof of \cref{lem:luen-pred} (Luenberger Prediction Error)}
\label{app:lb-error}

\begin{proof}
Define the state estimation error
$e_t^{\mathrm{LB}}(L) := x_t - \hat{x}_t^{\mathrm{LB}}(L)$.
From the system dynamics \cref{eq:po-lds} and the Luenberger update
\cref{eq:luen-dynamics},
\begin{align*}
    e_{t+1}^{\mathrm{LB}}
    &= Ax_t + w_t
       - \bigl(A_L\hat{x}_t^{\mathrm{LB}} + Ly_t\bigr) \\
    &= A_L(x_t - \hat{x}_t^{\mathrm{LB}})
       + w_t - L(y_t - Cx_t) \\
    &= A_L\,e_t^{\mathrm{LB}} + w_t - Lv_t.
\end{align*}
With $\hat{x}_0^{\mathrm{LB}} = 0$ and
$x_0=0$, we have $e_0^{\mathrm{LB}} = 0$.
Unrolling $e_t$ gives
\[
    e_t^{\mathrm{LB}} = \sum_{s=0}^{t-1} A_L^s(w_{t-1-s} - Lv_{t-1-s}).
\]
By strong stability ($\|A_L^k\| \le \kappa(1-\gamma)^k$, $\|L\| \le \kappa$) and the bounded disturbance assumption,
\begin{align*}
    \|e_t^{\mathrm{LB}}\|
    &\le \kappa(C_w + \kappa C_v)
         \sum_{s=0}^{t-1}(1-\gamma)^s
        = \frac{\kappa(C_w + \kappa C_v)}{\gamma}.
\end{align*}
The Luenberger prediction error is $y_t - \hat{y}_t^{\mathrm{LB}}(L) = Ce_t^{\mathrm{LB}}(L) + v_t$, so
\begin{align*}
    \|y_t - \hat{y}_t^{\mathrm{LB}}(L)\|
    &\le \|C\|\,\|e_t^{\mathrm{LB}}\| + C_v \\
    &\le \frac{\|C\|\kappa(C_w + \kappa C_v)}{\gamma}
       + C_v
    =: C_{\mathrm{pred}}.
\end{align*}
This bound is uniform in $t$ and independent of the disturbance realization.
\end{proof}

\subsection{Proof of \cref{thm:total-regret}
(Residual-Scaled Luenberger Regret of FM-POLS)}
\label{app:total-regret}

\begin{proof}
Fix any $L \in \mathcal{L}_{\kappa,\gamma}$. We decompose the Luenberger regret \cref{eq:luen-reg} by inserting the truncated predictor $\hat{y}_t^{\mathrm{TR}}(L) = M_L z_t$:
\[
\small
    \Reg_T^{\mathrm{LB}}(L)
    = \underbrace{
      \sum_{t=1}^T \!\left(\|\hat{y}_t - y_t\|^2
      - \|\hat{y}_t^{\mathrm{TR}}(L) - y_t\|^2\right)
      }_{= \Reg_T(M_L)}
    +\; \underbrace{
      \sum_{t=1}^T \!\left(\|\hat{y}_t^{\mathrm{TR}}(L)
      - y_t\|^2
      - \|\hat{y}_t^{\mathrm{LB}}(L) - y_t\|^2\right)
      }_{=:\;E_{\mathrm{trun}}}.
\]

\textbf{Bounding regression regret.}
Since $\hat{y}_t^{\mathrm{TR}}(L) = M_L z_t$ and FM-POLS
predicts $\hat{y}_t = M_t^{\mathrm{pols}} z_t$ with
features $z_t \in \mathbb{R}^{pH}$,
\cref{thm:pols-reg} gives
\begin{equation}\label{eq:Rlearn}
    \Reg_T(M_L)
    \le \lambda\|M_L\|_F^2
    + \Delta_{\max}^2\, pH\,
      \log\!\left(1
      + \frac{\sum_{t=1}^T\|z_t\|^2}{\lambda\,pH}
      \right).
\end{equation}

We bound the log argument. By \cref{lem:poly-growth},
$\|y_s\| \le C_y(1+s)^r$, so
$\|z_t\|^2 = \sum_{k=1}^H \|y_{t-k}\|^2
\le H C_y^2 (1+t)^{2r}$.
Summing over $t$ gives
\begin{equation}\label{eq:log-arg}
    \frac{\sum_{t=1}^T \|z_t\|^2}{\lambda\,pH}
    \le \frac{C_y^2 \sum_{t=1}^T (1+t)^{2r}}{\lambda\,p}
    \le \frac{C_y^2 (1+T)^{2r+1}}{\lambda\,p}.
\end{equation}

We bound the comparator norm. By strong stability
($\|A_L^k\| \le \kappa(1-\gamma)^k \le \kappa$,
$\|L\| \le \kappa$),
\begin{equation}\label{eq:ML-norm}
\begin{aligned}
    \|M_L\|_F^2
    &= \sum_{k=0}^{H-1} \|C A_L^k L\|_F^2\\
    &\le p \sum_{k=0}^{H-1} \|C A_L^k L\|^2 \\
    &\le p \|C\|^2 \kappa^4 \sum_{k=0}^{H-1} (1-\gamma)^{2k}\le \frac{\|C\|^2 \kappa^4 p}{\gamma}.
\end{aligned}
\end{equation}

\textbf{Bounding truncation deviation.}
Write $\hat{y}_t^{\mathrm{TR}}(L) - y_t
= (\hat{y}_t^{\mathrm{LB}}(L) - y_t)
- (\hat{y}_t^{\mathrm{LB}}(L) - \hat{y}_t^{\mathrm{TR}}(L))$.
Denote for brevity
\[
    a_t := \hat{y}_t^{\mathrm{TR}}(L) - y_t, \qquad
    p_t := \hat{y}_t^{\mathrm{LB}}(L) - y_t, \qquad
    \tau_t := \hat{y}_t^{\mathrm{LB}}(L)
              - \hat{y}_t^{\mathrm{TR}}(L),
\]
so that $a_t = p_t - \tau_t$. The truncation deviation expands as
\begin{align}
    E_{\mathrm{trun}}
    = \sum_{t=1}^T
       \bigl(\|p_t - \tau_t\|^2 - \|p_t\|^2\bigr)
    = \sum_{t=1}^T
       \bigl(\|\tau_t\|^2
       - 2\langle \tau_t,\, p_t \rangle\bigr).
       \label{eq:Rtrun-expand}
\end{align}

We bound each term separately.
By \cref{lem:luen-trunc} with the stated choice of $H$,
\begin{equation}\label{eq:term1}
    \sum_{t=1}^T \|\tau_t\|^2
    \le C_{\mathrm{trun}}^2.
\end{equation}
By \cref{lem:luen-pred}, $\|p_t\| \le C_{\mathrm{pred}}$
uniformly for all $t \ge 1$. Therefore,
\begin{equation}\label{eq:term2}
    \left|\sum_{t=1}^T
    \langle \tau_t,\, p_t \rangle\right|
    \le \sum_{t=1}^T \|\tau_t\|\,\|p_t\|
    \le C_{\mathrm{pred}} \sum_{t=1}^T \|\tau_t\|
    \le C_{\mathrm{pred}}\, C_{\mathrm{trun}},
\end{equation}
where the last inequality uses \cref{lem:luen-trunc}.
Substituting \cref{eq:term1,eq:term2} into
\cref{eq:Rtrun-expand} yields
\begin{equation}\label{eq:Rtrun-final}
    E_{\mathrm{trun}}
    \le C_{\mathrm{trun}}^2
    + 2C_{\mathrm{pred}}\, C_{\mathrm{trun}}.
\end{equation}

\textbf{Combining.}
Adding \cref{eq:Rlearn} and \cref{eq:Rtrun-final}, and
substituting \cref{eq:log-arg,eq:ML-norm} establishes \cref{thm:total-regret}
\begin{align*}
    \Reg_T^{\mathrm{LB}}(L)
    &\le \frac{\lambda\|C\|^2 \kappa^4 p}{\gamma}
      + \Delta_{\max}^2\, pH\,
        \log\!\left(1 + \frac{C_y^2(1+T)^{2r+1}}
        {\lambda\,p}\right)
      + C_{\mathrm{trun}}^2
      + 2C_{\mathrm{pred}}\, C_{\mathrm{trun}}.
\end{align*} 
\end{proof}

\section{Omitted Details for \cref{sec:hint_design}}
\label{app:section4}

\subsection{Proof of \cref{lem:residual-decomp} (Residual Decomposition)}
\label{app:residual-decomp}

\begin{proof}
From $x_{t+1} = A x_t + w_t$ and $y_t = C x_t + v_t$,
\[
    y_t = C A^t x_0 + C\sum_{k=0}^{t-1} A^k\,w_{t-1-k} + v_t.
\]
Substituting into $\Delta_t = \sum_{i=0}^m c_i y_{t-i}$ (with $x_0=0$) results in
\begin{equation}\label{eq:delta-substitute}
    \Delta_t
    = \sum_{i=0}^m c_i v_{t-i}
    + C\sum_{i=0}^m c_i \sum_{k=0}^{t-i-1} A^k\,w_{t-i-1-k}.
\end{equation}


For the second term, re-indexing the double sum by $s = k + i$
(the ``age'' of the disturbance) gives
\[
    C\sum_{i=0}^m c_i \sum_{k=0}^{t-i-1} A^k w_{t-i-1-k}
    = C\sum_{s=0}^{t-1}\left(\sum_{i=0}^{\min(s,m)} c_i A^{s-i}\right) w_{t-1-s}.
\]
Split the outer sum at $s=m$:
\begin{itemize}
    \item For $0\le s \le m-1$, the inner coefficient is $\sum_{i=0}^{s} c_i A^{s-i}$, giving the \textit{short-memory} term 
    \[
    C\sum_{s=0}^{m-1}\Bigl(\sum_{i=0}^s c_i A^{s-i}\Bigr) w_{t-1-s}.
    \]
    \item For $m\le s \le t-1$, the inner coefficient is $\sum_{i=0}^m c_i A^{s-i} = q(A)A^{s-m}$, giving the \textit{infinite-memory} term 
    \[
        C\sum_{s=m}^{t-1} q(A) A^{s-m} w_{t-1-s}
        = C\sum_{s=0}^{t-m-1} q(A) A^s w_{t-m-1-s}.
    \]
\end{itemize}
\end{proof}

\subsection{Proof of \cref{lem:2lag-operator} (Filtered Jordan Bound: Diagonalizable Case)}
\label{app:filtered-jordan-proof}

\begin{proof}
Diagonalizing $A = SJS^{-1}$ where $J$ is diagonal ($r = 1$), we have 
\[
\|(A^2-I)A^s\| = \|S(J^2-I)J^sS^{-1}\| \le \kappa_A \|(J^2-I)J^s\|.
\]
Because $(J^2-I)J^s$ is diagonal, its operator norm is simply the maximum absolute value of its diagonal entries: $\|(J^2-I)J^s\|\le \max_{i} |(\lambda_i^2-1)\lambda_i^s|$.
Summing this over time, we bound the maximum by the sum over all $n$ eigenvalues
\[
    \sum_{s=0}^{\infty} \max_{1 \le i \le n} \left|(\lambda_i^2-1)\lambda_i^s\right|
    \le \sum_{s=0}^{\infty} \sum_{i=1}^n \left|(\lambda_i^2-1)\lambda_i^s\right|
    = \sum_{i=1}^n \sum_{s=0}^{\infty} \left|(\lambda_i^2-1)\lambda_i^s\right|.
\]
For each eigenvalue $\lambda_i \in [-1, 1]$, we evaluate the inner infinite sum by considering two cases
\begin{itemize}
    \item If $|\lambda_i| < 1$, the geometric series converges, yielding
    \[
        \sum_{s=0}^{\infty} \left|(1-\lambda_i^2)\lambda_i^s\right| = (1-\lambda_i^2) \sum_{s=0}^{\infty} |\lambda_i|^s = \frac{1-\lambda_i^2}{1-|\lambda_i|} = 1 + |\lambda_i| < 2.
    \]
    \item If $|\lambda_i| = 1$, the prefactor $(1-\lambda_i^2)$ is exactly zero, so every summand vanishes and the infinite sum equals $0$.
\end{itemize}
In either case, the sum is bounded by $2$.
Because this bound holds for each eigenvalue, summing over all $n$ eigenvalues and multiplying by $\kappa_A$ yields the final bound.
\end{proof}

\subsection{Proof of \cref{lem:high-order-operator} (Gap-Free Filtered Jordan Bound)}
\label{app:high-order-proof}

\begin{proof}
The proof has two parts.
Marginal Jordan blocks are annihilated exactly by $(J_\lambda^2-I)^r$, while stable blocks contribute a uniform constant depending only on $r$.

\textbf{Reduce to individual blocks.}
With the Jordan decomposition $A=SJS^{-1}$ and $\kappa_A\!:=\!\|S\|\|S^{-1}\|$,
\[
    \|(A^2-I)^rA^s\|= \|S(J^2-I)^rJ^sS^{-1}\| 
    \le
    \kappa_A\|(J^2-I)^rJ^s\|.
\]

Since $J$ is block diagonal (with Jordan blocks $J_{\lambda_{i}}$), $\|(J^2 - I)^r J^s\| = \max_i \|(J_{\lambda_i}^2 - I)^r J_{\lambda_i}^s\|$. 
Bounding the block maximum by the sum and exchanging the order of summation gives
\begin{equation}
\label{eq:filtered-jordan-bound}
    \sum_{s=0}^{\infty}\|(A^2-I)^rA^s\|
    \le
    \kappa_A
    \sum_i
    \sum_{s=0}^{\infty}
    \|(J_{\lambda_i}^2-I)^rJ_{\lambda_i}^s\|.
\end{equation}

Since the number of Jordan blocks is at most $n$, it suffices to prove a uniform bound for each block: $\sum_s \|(J_\lambda^2 - I)^r J_\lambda^s\| \le K_r$ for each Jordan block $J_\lambda=\lambda I+N$ where $\lambda\in[-1,1]$.
Since we assume each block has size at most $r$, the nilpotent part satisfies $N^r=0$.

\textbf{Marginal blocks vanish.}
For marginal blocks with $\lambda \in \{-1, +1\}$, using $\lambda^2 = 1$,
\[
J_\lambda^2 - I = \lambda^2 I + 2\lambda N + N^2 - I = N(2\lambda I + N).
\]
Since $N$ and $2\lambda I + N$ commute and $N^r=0$,  
$
(J_\lambda^2 - I)^r = N^r(2\lambda I + N)^r = 0,
$
Therefore 
\[
(J_\lambda^2 - I)^r J_\lambda^s = 0
\]
for all $s$, so marginal blocks contribute zero to \cref{eq:filtered-jordan-bound}.

\textbf{Stable blocks.}
Now suppose $|\lambda| < 1$. 
Define $f_s(z) := (z^2 - 1)^r z^s$, so that $f_s(J_\lambda) = (J_\lambda^2 - I)^r J_\lambda^s$.
Taylor's formula expanded at $\lambda$ gives the scalar polynomial identity
\[
    f_s(z) = \underbrace{\sum_{\ell=0}^{r-1} \frac{f_s^{(\ell)}(\lambda)}{\ell!}(z - \lambda)^\ell}_{\mathtt{degree~(r-1)~Taylor~polynomial}} + \underbrace{(z - \lambda)^r q_s(z)}_{\mathtt{remainder}}
\]
for some polynomial $q_s$.
Evaluating this identity at $z \!=\! J_\lambda$ (which is
valid since polynomial identities are preserved under matrix evaluation) using $J_\lambda-\lambda I=N$ and $N^r = 0$, the remainder vanishes
\[
f_s(J_\lambda) = \sum_{\ell=0}^{r-1} \frac{f_s^{(\ell)}(\lambda)}{\ell!} N^\ell.
\]
Since $N$ is the nilpotent shift, $\|N^\ell\|\le1$.
Therefore $\|f_s(J_\lambda)\| \le \sum_{\ell=0}^{r-1} \tfrac{|f_s^{(\ell)}(\lambda)|}{\ell!}$.
Our goal is to bound the quantity
\[
    \Phi(\lambda) := \sum_{i=0}^\infty \|(J_\lambda^2 - I)^r J_\lambda^s\| = \sum_{s=0}^\infty \|f_s(J_\lambda)\| \le \sum_{\ell=0}^{r-1} \frac{\Sigma_\ell(\lambda)}{\ell!},
    \qquad
    \Sigma_\ell(\lambda):=\sum_{s=0}^\infty |f_s^{(\ell)}(\lambda)|
\]
It remains to bound $\Sigma_\ell(\lambda)$ uniformly over $\lambda\in(-1,1)$. 
Write $f_s=g\cdot h_s$, where $g(z):=(z^2-1)^r$ is a fixed polynomial
independent of $s$ and $h_s(z):=z^s$ carries the time index.
By Leibniz's rule (product rule for the $\ell$-th derivative),
\[
    f_s^{(\ell)}(\lambda)
    =
    \sum_{a=0}^{\ell}
    \binom{\ell}{a}
    g^{(a)}(\lambda)h_s^{(\ell-a)}(\lambda).
\]
For $q\ge0$,
\[
    h_s^{(q)}(\lambda)
    =
    \frac{s!}{(s-q)!}\lambda^{s-q}
\]
for $s\ge q$, and is zero otherwise.

\textit{Diverging near marginality.}
Summing $|h_s^{(q)}(\lambda)|$ over $s$ by the generalized geometric
series $\sum_{u \ge 0} \binom{u + q}{q} |\lambda|^u = (1 - |\lambda|)^{-(q+1)}$, we get
\begin{equation}
    \label{eq:bound-h}
    \sum_{s=0}^{\infty}|h_s^{(q)}(\lambda)|
    =
    \frac{q!}{(1-|\lambda|)^{q+1}}, \qquad \lambda\in(-1,1).
\end{equation}

This term blows up as $|\lambda|$ goes to $1$, with rate controlled by the exponent $q+1 = \ell - a + 1$.
We next show that $|g^{(a)}(\lambda)|$ is small enough to cancel this blowup for every admissible $(\ell, a)$.

\textit{Vanishing near marginality.}
Let $\delta := 1 - |\lambda|$.
Since $g(z) = (z-1)^r(z+1)^r$, the polynomial $g$ has a zero of order $r$ at both $1$ and $-1$.
Consequently, for each \(0\le a\le r-1\), there exists
\(M_{r,a}<\infty\), depending only on \(r\) and \(a\), such that
\begin{equation}
    \label{eq:bound-g}
    |g^{(a)}(\lambda)|
    \le
    M_{r,a}\delta^{r-a},
    \qquad \lambda\in(-1,1).
\end{equation}
For example, differentiating \((z-1)^r(z+1)^r\) by Leibniz's rule shows that
\[
    g^{(a)}(z) = \sum_{b=0}^a \binom{a}{b} \frac{r!}{(r-b)!}(z-1)^{r-b} \frac{r!}{(r-a+b)!}(z+1)^{r-a+b}.
\]
and thus one may take
\begin{equation}
    \label{eq:M-const}
    M_{r,a}
    =
    2^r
    \sum_{b=0}^{a}
    \binom{a}{b}
    \frac{r!}{(r-b)!}
    \frac{r!}{(r-a+b)!}.
\end{equation}
Indeed, a bound for $|g^{(a)}(\lambda)|$ contains factors $|\lambda-1|^{r-b}$, $|\lambda+1|^{r-a+b}$.
For $\lambda\in(-1,1)$, at least one of $|\lambda - 1|$, $|\lambda + 1|$ equals $\delta$ and the other lies in $[1, 2)$ (specifically, $|\lambda - 1| = \delta$ when $\lambda \ge 0$, and $|\lambda + 1| = \delta$ when $\lambda \le 0$), while the other is at most \(2\).
Therefore, one of the factors $|\lambda - 1|^{r-b}$, $|\lambda + 1|^{r-a+b}$ is a power of $\delta$ with exponent at least $r-a$ (since $b \le a$ and $r-a+b \ge r-a$), and the other is bounded by $2^r$ (since both exponents are at most $r$). Using $\delta \in (0,1]$ to absorb any extra power of $\delta$ beyond $\delta^{r-a}$ and summing over $b$ then gives the displayed bound in \cref{eq:M-const}.

\textbf{Cancellation and final bound.}
Combining \cref{eq:bound-h,eq:bound-g}, for every \(0\le\ell\le r-1\),
\[
\begin{aligned}
    \Sigma_\ell(\lambda)
    &\le
    \sum_{a=0}^{\ell}
    \binom{\ell}{a}
    |g^{(a)}(\lambda)|\cdot
    \sum_{s=0}^{\infty}
    |h_s^{(\ell-a)}(\lambda)|  \\
    &\le
    \sum_{a=0}^{\ell}
    \binom{\ell}{a}
    M_{r,a}\delta^{r-a}
    \frac{(\ell-a)!}{\delta^{\ell-a+1}}  \\
    &=
    \delta^{r-\ell-1}
    \sum_{a=0}^{\ell}
    \binom{\ell}{a}
    M_{r,a}(\ell-a)! .
\end{aligned}
\]
Here the dependence on \(a\) cancels in the exponent of \(\delta\). This is where the prefactor $(z^2 - 1)^r$ does its work.
Since \(\ell\le r-1\), we have \(\delta^{r-\ell-1}\le1\) and thus
\[
    \Sigma_\ell(\lambda)
    \le
    K_{r,\ell},
    \qquad
    K_{r,\ell}
    :=
    \sum_{a=0}^{\ell}
    \binom{\ell}{a}
    M_{r,a}(\ell-a)!,
\]
where \(K_{r,\ell}\) depends only on \(r\) and \(\ell\), independent of \(\lambda\).
Therefore,
\[
    \Phi(\lambda)
    \le
    \sum_{\ell=0}^{r-1}
    \frac{K_{r,\ell}}{\ell!}
    =:K_r
\]
uniformly over all stable blocks. Since marginal blocks contribute zero, the
same bound holds for every Jordan block. Aggregating over at most \(n\) blocks
in \cref{eq:filtered-jordan-bound} yields
\[
    \sum_{s=0}^{\infty}\|(A^2-I)^rA^s\|
    \le
    \kappa_A n K_r.
\]
This completes the proof.
\end{proof}

\subsection{Proof of \cref{thm:jordan-hint} (Bounded Residual under High-Order Differencing)}
\label{app:high-order-thm-proof}
\begin{proof}
With $q_r(z) = (z^2 - 1)^r$ of degree $2r$, \cref{lem:residual-decomp} gives the residual decomposition
\[
\Delta_t = \underbrace{\sum_{k=0}^r (-1)^k \binom{r}{k} v_{t-2k}}_{\mathtt{measurement~noise~(a)}} + \underbrace{C \sum_{s=0}^{2r-1} \Bigl(\sum_{i=0}^s c_i A^{s-i}\Bigr) w_{t-1-s}}_{\mathtt{short~memory~(b)}} + \underbrace{C \sum_{s=0}^{t-2r-1} (A^2 - I)^r A^s w_{t-2r-1-s}}_{\mathtt{infinite~memory~(c)}},
\]
where $c_i$ are the coefficients of $q_r$ and $\|q_r\|_1 := \sum_i |c_i| = \sum_{k=0}^r \binom{r}{k} = 2^r$.
We bound each term.

\textit{Measurement noise.} The triangle inequality gives $\|\mathtt{(a)}\| \le \sum_{k=0}^r \binom{r}{k} C_v = 2^r C_v$.

\textit{Short memory.} Each of the $2r$ terms has coefficient $\sum_{i \le s} c_i A^{s-i}$ with $\sum_{i \le s} |c_i| \le \|q_r\|_1 = 2^r$. By \cref{lem:poly-growth}, $\|A^{s-i}\| \le \kappa_A (1 + s - i)^{r-1} \le \kappa_A (2r)^{r-1}$ for $s - i \le 2r - 1$. Hence
\begin{equation*}
    \|\mathtt{(b)}\| \le H_r \kappa_A \|C\| C_w, \quad H_r := 2r \cdot 2^r \cdot (2r)^{r-1} = (4r)^r.
\end{equation*}

\textit{Infinite memory.} By the triangle inequality and \cref{lem:high-order-operator},
\[
\|\mathtt{(c)}\| \le \|C\| C_w \sum_{s=0}^\infty \|(A^2 - I)^r A^s\| \le K_r\, n\, \kappa_A\, \|C\|\, C_w.
\]
Combining the three contributions completes the proof.
\end{proof}

\subsection{Explicit Bounds on the Constants $H_r$ and $K_r$}
\label{app:explicit-constants}

In the proofs of \cref{thm:jordan-hint}  (Appendix~\ref{app:high-order-thm-proof}) and \cref{lem:high-order-operator} (Appendix~\ref{app:high-order-proof}), we introduce a crude bound on $H_r$ and an implicit bound on $K_r$. Here we derive tighter or more explicit bounds for them.

\textbf{Bound on $H_r$.}
We sharpen the bound on the short-memory term by using the binomial form of $\|A^j\|$ instead of the cruder $(1+j)^{r-1}$.
In particular, from \cref{eq:Jordan-power-bound} in the proof of \cref{lem:poly-growth}, we have the following tighter bound
\[
    \|A^j\| \le \kappa_A \sum_{\ell=0}^{r-1}\binom{j}{\ell}.
\]
Applying this to the short-memory term, i.e., term $\mathtt{(b)}$ in Appendix~\ref{app:high-order-thm-proof}, we obtain
\[
\begin{aligned}
    \|\mathtt{(b)}\|
    &\le \|C\| C_w \sum_{s=0}^{2r-1}\Bigl\|\sum_{i=0}^s c_i A^{s-i}\Bigr\| \\
    &\le \kappa_A \|C\| C_w \sum_{s=0}^{2r-1}\sum_{i=0}^s |c_i| \sum_{\ell=0}^{r-1}\binom{s-i}{\ell} \\
    &\le \kappa_A \|C\| C_w \,\|q_r\|_1 \sum_{j=0}^{2r-1}\sum_{\ell=0}^{r-1}\binom{j}{\ell},
\end{aligned}
\]
where the last step substitutes $j = s - i$ and extends the range of $j$ to $[0, 2r-1]$. Using $\|q_r\|_1 = 2^r$ and the identity $\sum_{j=0}^{2r-1}\binom{j}{\ell} = \binom{2r}{\ell+1}$,
\[
    \|\mathtt{(b)}\|
    \le \kappa_A \|C\| C_w \cdot 2^r \sum_{\ell=0}^{r-1}\binom{2r}{\ell+1}
    = \kappa_A \|C\| C_w \cdot 2^r \sum_{m=1}^{r}\binom{2r}{m}.
\]
Since $\sum_{m=0}^{2r}\binom{2r}{m} = 4^r$, we have $\sum_{m=1}^r \binom{2r}{m} \le 4^r$, and therefore
\[
    H_r \le 2^r \cdot 4^r = 8^r.
\]
This refines the crude bound $H_r \le (4r)^r$ in Appendix~\ref{app:high-order-thm-proof}; both have base $8$ at $r=2$, but $8^r$ has a constant base while $(4r)^r$ grows superexponentially.

\textbf{Bound on $K_r$.}
For \(0 \le b \le a \le r-1\), we have
\[
    \frac{r!}{(r-b)!}\frac{r!}{(r-a+b)!}
    \le r^b \cdot r^{a-b}
    = r^a.
\]
Substituting into $M_{r,a}$ in \cref{eq:M-const} gives
\[
    M_{r,a}
    \le 2^r r^a \sum_{b=0}^a \binom{a}{b}
    = 2^r r^a \cdot 2^a
    = 2^r (2r)^a.
\]
Consequently,
\[
\begin{aligned}
    K_{r,\ell}
    &:= \sum_{a=0}^\ell \binom{\ell}{a} M_{r,a}(\ell-a)! \le 2^r \sum_{a=0}^\ell \binom{\ell}{a}(2r)^a (\ell-a)! = 2^r \ell! \sum_{a=0}^\ell \frac{(2r)^a}{a!}
    \le 2^r e^{2r} \ell!.
\end{aligned}
\]
Therefore
\[
    K_r := \sum_{\ell=0}^{r-1}\frac{K_{r,\ell}}{\ell!}
    \le \sum_{\ell=0}^{r-1} 2^r e^{2r}
    = r(2e^2)^r.
\]

\textbf{Combined bound.}
The filtered residual bound \cref{eq:jordan-bound} can therefore be written explicitly as
\[
    \Delta_{\max}
    \le 2^r C_v + \kappa_A \|C\| C_w \bigl(8^r + r(2e^2)^r\, n\bigr).
\]
The dependence is polynomial in $n$ and $\kappa_A$, exponential in $r$, and independent of the spectral gap.

\section{Additional Experiments}
\label{app:additional-exp}

This appendix reports three further experiments that complement \cref{sec:experiments}: a stochastic sanity check against the Kalman comparator (Exp-A1), higher-order Jordan blocks with high-order differencing hints (Exp-A2), and the limitation of model-free hints under complex marginal eigenvalues (Exp-A3).

\textbf{Noise model and Kalman comparator.} 
Exp-A1 and Exp-A2 use i.i.d.\ Gaussian disturbances with
\begin{equation}
    \label{eq:noise-model-gaussian}
    w_t \sim \mathcal{N}(0, C_w^2 I),\quad v_t \sim \mathcal{N}(0, C_v^2 I),
\end{equation}
where the Kalman filter is the canonical optimal predictor \textit{in expectation}. 
Accordingly, we benchmark FM-POLS against the Kalman comparator in this case, following the convention of stochastic LDS prediction works~\cite{ghai2020no, tsiamis2022online}.
Note that this is a single fixed comparator across all time steps and noise realizations, in contrast to the running best-in-hindsight Luenberger comparator $L^\star(t) \in \argmin_{L} \sum_{s=1}^t \|y_s - \hat y_s^{\mathrm{LB}}(L)\|^2$ used in \cref{sec:experiments}.
Results are averaged over $50$ Monte Carlo trials.
On the other hand, Exp-A3 uses the nonstochastic noise model \cref{eq:noise-model} and benchmarks against $L^\star(t)$.

\begin{figure}[t]
    \centering
    \includegraphics[width=1\linewidth]{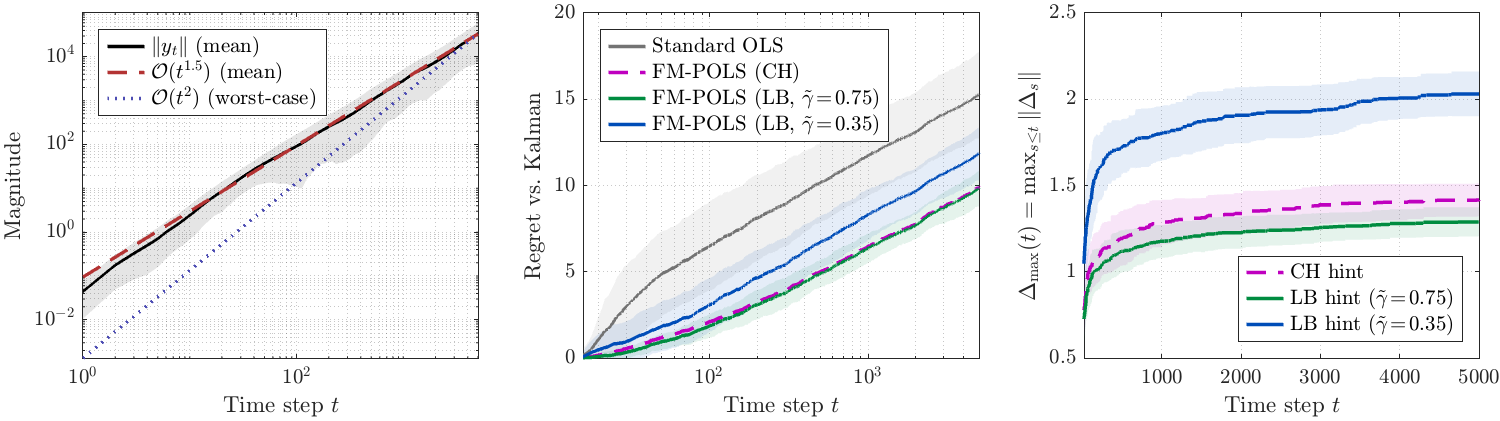}
    \caption{Exp-A1: Double integrator under Gaussian noise.
    Left: Polynomial signal growth for $y_t$ (log-log plot), with $\|y_t\|$ tracking $\mathcal{O}(t^{1.5})$ in expectation (slower than the worst-case $\mathcal{O}(t^2)$ shown for reference).
    Middle: Regret against the Kalman filter, growing logarithmically (linear growth in log-linear plot).
    Right: Bounded running max hint residual $\Delta_{\max}(t)$. Shaded regions show $\pm 1$ std across $50$ Monte Carlo trials.}
    \label{fig:exp-A1}
\end{figure}

\textbf{Exp-A1: Logarithmic Regret against Kalman Filter.}
We consider the double integrator
$A = \left[\begin{smallmatrix}1 & 1\\0 & 1\end{smallmatrix}\right]$ ($r=2$) ,
$C = [1,\; 0]$, under Gaussian noise with $C_w = 0.2$, $C_v = 0.05$ in \cref{eq:noise-model-gaussian}.
We compare standard OLS, FM-POLS with the Cayley--Hamilton (CH) hint $q_{\mathrm{ch}}(z) = (z-1)^2$, and FM-POLS with two Luenberger hints at $\tilde\gamma \in \{0.25, 0.65\}$, using $H = 15$, $\lambda = 1$, $T=5000$.
\cref{fig:exp-A1}~(middle) shows that all FM-POLS variants achieve logarithmic regret against the Kalman filter (linear on the log time axis). 
As in Exp1 (\cref{sec:experiments}), the aggressive Luenberger hint ($\tilde\gamma = 0.25$)
achieves lower regret than the conservative one ($\tilde\gamma = 0.65$), and the CH hint is competitive with the Luenberger hints; \cref{fig:exp-A1}~(right) shows that the aggressive gain achieves a smaller running max hint residual $\Delta_{\max}(t)$.

\begin{figure}[t]
    \centering
    \includegraphics[width=1\linewidth]{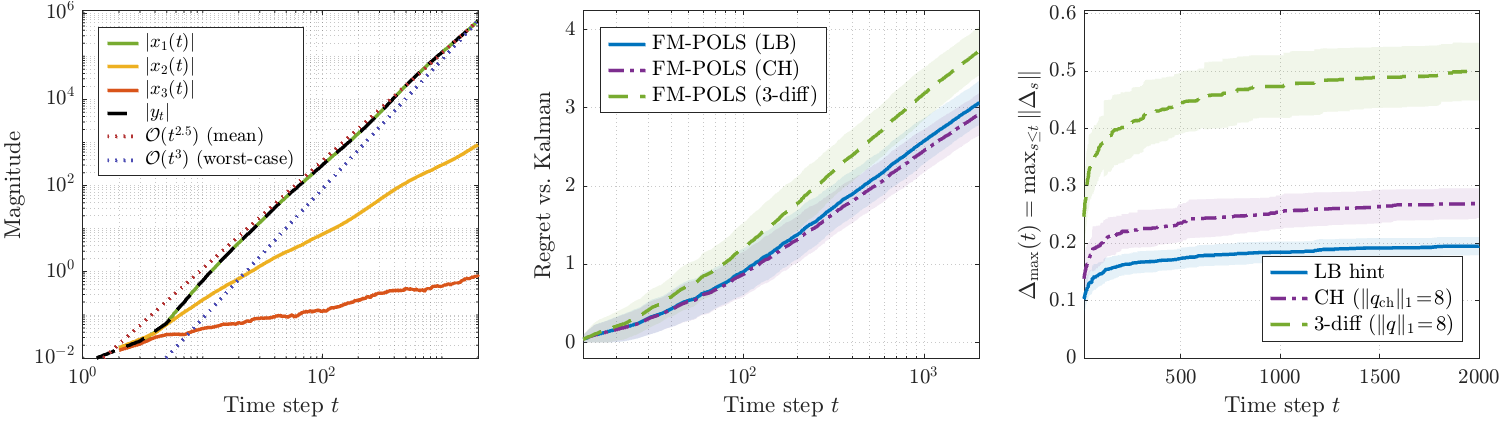}
    \caption{Exp-A2: Jordan block ($r=3$) under Gaussian noise. Left: Polynomial trajectory growth for $x_t$ and $y_t$ (log-log plot), with $\|y_t\|$ tracking $\mathcal{O}(t^{2.5})$ in expectation (slower than the worst-case $\mathcal{O}(t^3)$ shown for reference).
    Middle: Regret against the Kalman filter, growing logarithmically (linear growth in log-linear plot).
    Right: Bounded running max hint residual $\Delta_{\max}(t)$.
    Shaded regions show $\pm 1$ std across $50$ Monte Carlo trials.}
    \label{fig:exp-A2}
\end{figure}

\begin{figure}[t]
    \centering
    \includegraphics[width=1\linewidth]{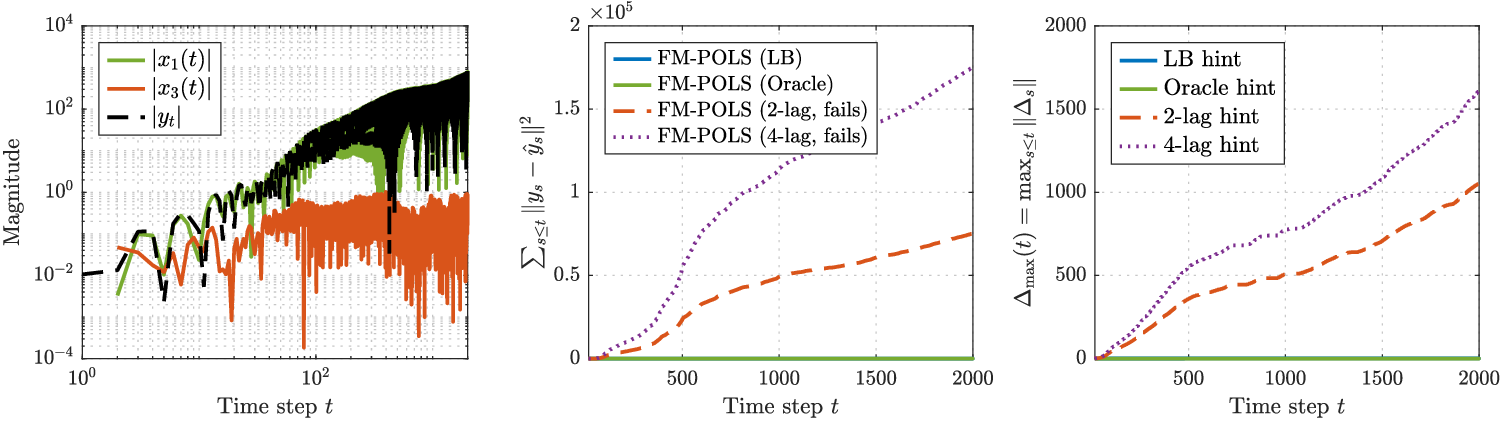}
    \caption{Exp-A3: Complex marginal eigenvalues. System with eigenvalues $e^{\pm i\theta}$ ($\theta = 0.7$) and Jordan size $2$ ($r = 2$), under the noise model in \cref{eq:noise-model}.
    Left: Polynomial signal growth, with $\|x_1(t
    )\|$ and $\|y_t\|$ at the worst-case rate $\mathcal{O}(t^2)$.
    Middle: Cumulative prediction loss. 
    Right: Running max hint residual $\Delta_{\max}(t)$.
    The model-free $2$-lag and $4$-lag hints fail with linearly growing $\Delta_{\max}$ and cumulative prediction loss, while the model-based Luenberger hint and oracle polynomial hint (which knows $\theta$ and $r$) achieve significantly smaller $\Delta_{\max}$ ones.
    }
    \label{fig:exp4}
\end{figure}

\textbf{Exp-A2: High-Order Jordan Blocks and Model-Free Differencing.}
\label{app:exp3}
We consider a $3\times 3$ Jordan block: $A \!=\! \left[\begin{smallmatrix}1&1&0\\0&1&1\\0&0&1\end{smallmatrix}\right]$ ($\lambda \!=\! 1$, $r \!=\! 3$), $C \!=\! [1,\;0,\;0]$, under Gaussian noise with $C_w = 0.02$, $C_v = 0.01$ in \cref{eq:noise-model-gaussian}.
The state and output grow as $\mathcal{O}(t^{2.5})$ in expectation, slower than the worst-case $\mathcal{O}(t^3)$ (\cref{fig:exp-A2}, left).
We compare three predictive hints with $H = 12$, $\lambda = 1$:
(i)~a model-based Luenberger hint; 
(ii)~the CH hint $q_{\mathrm{ch}}(z) = (z-1)^3$, of order $3$ with $\|q_{\mathrm{ch}}\|_1 = 8$, requiring knowledge of eigenvalues;
and (iii)~the model-free third-order differencing ($3$-diff) hint $q(z) = (z^2-1)^3$, of order $6$ with $\|q\|_1 = 8$, requiring only knowledge of $r$.
\cref{fig:exp-A2} (middle) shows that all three hints achieve logarithmic regret. 
The CH and $3$-diff hints have identical noise amplification factor of $8$, while they differ in memory requirement ($3$ vs.\ $6$ lags).
The running max residuals in \cref{fig:exp-A2} (right) stabilize for all three hints, with the Luenberger hint achieving the smallest $\Delta_{\max}$.
Notice that despite this, as shown in \cref{fig:exp-A2} (middle), the regret corresponding to the Luenberger hint is not uniformly smaller than the CH hint's, showing that $\Delta_{\max}$ alone does not determine regret.

\textbf{Exp-A3: Limits of Model-Free Hints for Complex Eigenvalues.}
\label{app:exp4}
We investigate the case of complex marginal eigenvalues, which remains an open problem for model-free hint design.
We use $C=[1, 0.3, 0, 0]$ and the $4 \times 4$ system $A = \left[\begin{smallmatrix}R_\theta & I_2\\ 0 & R_\theta\end{smallmatrix}\right]$ with $R_\theta$ the rotation matrix by $\theta = 0.7$, giving complex eigenvalues $e^{\pm i\theta}$ each with Jordan size $2$ ($r = 2$).
The state grows at the worst-case polynomial rate $\|x_t\| = \mathcal{O}(t^2)$, visible as a slope-$2$ line on the log-log plot in \cref{fig:exp4} (left).
We compare four predictive hints:
(i)~a Luenberger hint (model-based);
(ii)~the \emph{oracle} polynomial hint
$q(z) = (z^2 - 2\cos\theta \cdot z + 1)^2$, the minimal annihilating polynomial which requires spectral knowledge;
(iii)~the $2$-lag hint $q(z) = z^2 - 1$; and
(iv)~the $4$-lag hint $q(z) = z^4 - 1$. 
Both model-free polynomial hints (iii) and (iv) satisfy $q(e^{i\theta}) \ne 0$ for generic $\theta$ and are therefore expected to fail.
\cref{fig:exp4} (middle, right) confirms this: the $2$-lag and $4$-lag hints fail and produce linear growing residuals and cumulative prediction loss, while the model-based Luenberger and oracle hints achieve significantly smaller $\Delta_{\max}$ and loss.
This demonstrates that the universal differencing filters designed for real-spectrum systems fundamentally cannot handle complex marginal eigenvalues.

We summarize all our experiments in \cref{sec:experiments} and Appendix~\ref{app:additional-exp} in the table below.

\begin{table}[!htbp]
\centering
\small
\caption{Summary of experiments. $r$ = largest Jordan block size. Hint types: LB = Luenberger (model-based), CH = Cayley--Hamilton, $k$-diff = $k$-th order differencing.}
\label{tab:experiments}
\begin{tabular}{@{}lllllll@{}}
\toprule
\# & System & $r$ & Noise & Predictors (Hints) & Comparator & Metric \\
\midrule
Exp1   & Double integrator  & 2 & \cref{eq:noise-model} & OLS, CH, LB & $L^\star(t)$ & Regret, $\Delta_{\max}$ \\
Exp2   & Symmetric $\pm 1$            & 1 & \cref{eq:noise-model} & Kalman, $\mathcal{H}_\infty$, LB, 2-lag    & ---           & Cum.\ \& per-step loss \\
Exp3   & Double integrator & 2 & \cref{eq:noise-model} & LB (vary $H$, $\lambda$)                    & $L^\star(t)$ & Regret \\
Exp-A1 & Double integrator  & 2 & Gaussian  & OLS, CH, LB                                 & Kalman        & Regret, $\Delta_{\max}$ \\
Exp-A2 & Jordan $3 \times 3$          & 3 & Gaussian  & LB, CH, $3$-diff                              & Kalman        & Regret, $\Delta_{\max}$ \\
Exp-A3 & Rotation $e^{\pm i\theta}$   & 2 & \cref{eq:noise-model} & LB, oracle, 2-lag, 4-lag                    & ---   & Loss, $\Delta_{\max}$ \\
\bottomrule
\end{tabular}
\end{table}



\end{document}